\documentclass[twoside,11pt]{article}

\usepackage{jmlr2e}
\usepackage{graphicx}% Include figure files
\usepackage{graphics}
\usepackage{dcolumn}% Align table columns on decimal point
\usepackage{bm}% bold math
\usepackage{color}
\usepackage{caption}
\usepackage{subcaption}
\usepackage{amsmath}
\usepackage{amssymb}
\usepackage{wrapfig}

\usepackage[skip = 2pt, font = scriptsize]{caption}
\usepackage{algorithm}
\usepackage{algorithmic}
\usepackage{footnote}
\usepackage{booktabs}
\usepackage{color}

\newcommand{\mvec}[1]{\mathbf{#1}}
\newcommand{\Trans}{{^{\mathrm{T}}}}

{
	\newtheorem{assumption}{Assumption}
}
\jmlrheading{1}{2019}{1-48}{4/00}{10/00}{Granziol, Ru, Zohren, Dong, Roberts, Osborne}

\ShortHeadings{A Maximum Entropy approach to Massive Graph Spectra}{Granziol, Ru, Zohren, Dong, Roberts, Osborne}
\firstpageno{1}

\begin{document}
\title{A Maximum Entropy approach to Massive Graph Spectra}

{\author{Diego Granziol, Binxin Ru, Xiaowen Dong, Stefan Zohren, Michael Osborne, Stephen Roberts \\ \email diego@robots.ox.ac.uk, robin@robots.ox.ac.uk \\
%  \homepage{http://www.Second.institution.edu/~Charlie.Author}
\addr
 Machine Learning Research Group and Oxford-Man Institute for Quantitative Finance,\\ Department of Engineering Science, University of Oxford}}

%\collaboration{CLEO Collaboration}%\noaffiliation

\date{March 23, 2019}%\today}% It is always \today, today,
             %  but any date may be explicitly specified

\maketitle

\begin{abstract}
Graph spectral techniques for measuring graph similarity, or for learning the cluster number, require kernel smoothing. The choice of kernel function and bandwidth are typically chosen in an ad-hoc manner and heavily affect the resulting output. We prove that kernel smoothing biases the moments of the spectral density. We propose an information theoretically optimal approach to learn a smooth graph spectral density, which fully respects the moment information. Our method's computational cost is linear in the number of edges, and hence can be applied to large networks, with millions of nodes. We apply our method to the problems to graph similarity and cluster number learning, where we outperform comparable iterative spectral approaches on synthetic and real graphs.

% Graph spectra have been successfully used to classify network types, compute the similarity between graphs, and determine the number of communities in a network. For large graphs, where an eigen-decomposition is infeasible, iterative moment matched approximations to the spectra and kernel smoothing are typically used. We show that the underlying moment information is lost when using kernel smoothing. We further propose a spectral density approximation based on the method of Maximum Entropy, which matches moments exactly and is everywhere positive. We demonstrate its effectiveness and superiority over existing approaches in learning graph spectra, via experiments on both synthetic networks, such as the Erd\H{o}s-R\'{e}nyi and Barab\'{a}si-Albert random graphs, and real-world networks, such as biological and road networks as well as the social networks for YouTube from the SNAP dataset.
%\sz{I made several small changes to the abstract and the introduction.}
\end{abstract}


\section{Introduction: networks, their graph spectra and importance}\label{sec:intro}

Many systems of interest can be naturally characterised by complex networks; examples include social networks \citep{mislove2007measurement,flake2000efficient,leskovec2007dynamics}, biological networks \citep{palla2005uncovering} and technological networks. Trends, opinions and ideologies spread on a social network, in which people are nodes and edges represent relationships. 
% The World Wide Web is a complex network of documents with web pages representing nodes and hyper-links denoting edges. Neural networks, considered state of the art machine learning algorithms for a variety of complex problems, can be seen as directed networks where neurons are the nodes and the synaptic connections between them are the edges. A variety of complex networks have been studied in the literature, from scientific collaborations \citep{ding2011scientific}, ecological/cellular networks \citep{fath2007ecological}, to sexual contacts \citep{albert2002statistical}. For a comprehensive introduction we refer the reader to \citep{newman2010networks}.
% \subsection{Network spectra and applications}\label{subsec:net_spec_and_app}
Networks are mathematically represented by graphs. Of crucial importance to the understanding of the properties of a network or graph is its spectrum, which is defined as the eigenvalues of its adjacency or Laplacian matrix \citep{farkas2001spectra,cohen2018approximating}.
% The spectrum of a graph is the density of the eigenvalues of its adjacency/Laplacian matrix \citep{{farkas2001spectra},{cohen2018approximating}}. 
The spectrum  of a graph can be considered as a natural set of graph invariants and has been extensively studied in the fields of chemistry, physics and mathematics \citep{biggs1976graph}. Spectral techniques have been extensively used to characterise the global network structure \citep{newman2006modularity} and in practical applications thereof, such as facial recognition and computer vision \citep{belkin2003Laplacian}, learning dynamical thresholds \citep{mcgraw2008Laplacian}, clustering \citep{von2007tutorial}, and measuring graph similarity \citep{takahashi2012discriminating}. 

A major limitation in utilizing graph spectra to solve problems such as graph similarity and estimating the number of clusters\footnote{Just two example applications of the general method we propose for learning graph spectrum in this paper.} is the inability to automatically and consistently learn an everywhere-positive, non-singular approximation to the spectral density. Full eigen-decomposition, which is prohibitive for large graphs, or iterative moment-matched approximations both give a Dirac sum that must be smoothed to be everywhere positive. The choice of smoothing kernel $k_\sigma(x,x')$ and kernel bandwidth choice $\sigma$, or number of histogram bins, which are usually chosen in an ad-hoc manner, can significantly affect the resulting output. 
%We investigate the theoretical consequences of kernel smoothing and propose an information theoretically optimal method of spectral density approximation which respects the underlying moment information of the spectral density.
% In a nutshell, in order to make certain problems tractable, e.g., comparisons of network spectra, current methods loose the only exact information we have about the network. In this paper, we are interested in an efficient and accurate everywhere-positive approximation of the spectral density of large graphs.

% \subsection{Main contributions}\label{subsec:contributions}
%In view of the-above-mentioned problems, we propose a novel approach which efficiently learns the density of maximum entropy compatible with the moment information of the graph. This naturally leads to a continuous, everywhere postive approximation of the graph spectrum, thus avoiding the need for kernel smoothing and any limitations caused by that. 
%Our approach is grounded in the observation that for random graphs, the finite size effects increase for the higher order spectral moments and furthermore that for any iterative method, we are in effect creating a surrogate spectral density which matches the lower order moments.
%We apply our method to solve challenging graph problems such as network classification and the cluster number detection.
The main contributions of this paper are as follows:
\begin{itemize}
    %\item We show that low order spectral moments of random graphs are least affected by finite size effects and hence more informative about the global underlying process than their higher order counterparts
    %\item We show that all iterative methods learn are moment matched approximations to the underlying spectral density, that this moment information is consistent to the underlying stochastic process
    \item We prove that the method of kernel smoothing, commonly used in methods to visualize and compare graph spectral densities, biases moment information;
    %\item We propose a highly computationally efficient method based on the method of Maximum Entropy and stochastic trace estimation, for which we develop a novel algorithm with cost $\mathcal{O}(n_\mathrm{nz} \times m \times d)$ whilst that of Lanczos algorithm is $\mathcal{O}(n_\mathrm{nz}\times m \times d + nm^{2})$ where $n_\mathrm{nz}$ is the number of non-zero entries in the graph, $n$ is the number of nodes in the graph, $m$ is the number of moments(or iterative steps) used and $d$ is the number of stochastic trace vectors used. 
    \item We propose a computationally efficient and information theoretically optimal smooth spectral density approximation, based on the method of Maximum Entropy, which fully respects the moment information. It further admits analytic forms for symmetric and non-symmetric KL-divergences and Shannon entropy; %This is because the density of maximum entropy is the flattest possible density on thebounded domain compatible with the moment information given and so can be seen as maximally smooth whilstcompatible with the observed data.
    %\item Our proposed method learns the maximum entropy spectral approximation to the graph, enabling us to analytically compute/approximate the shannon entropy and von-neumann entropy of the graph as well as the relative entropy between two graphs.
    \item We utilize our information theoretic spectral density approximation, on two example applications. We investigate graph similarity and to learn the number of clusters in a graph, outperforming iterative smoothed spectral approaches on both real and synthetic data-sets
    
    % \item We show empirically and justify using results from random graphs, that our method is able to learn the underlying stochastic process of a network and can be utilized for computing the similarity among networks from a wide range of synthetic and real world datasets;
    % % We develop theoretical bounds on the number of clusters in a fully connected graph from the eigenspectrum depending on the number intra cluster to inter cluster connections and Weyl`s bound on Hermitian matrices. 
    % \item We develop bounds on changes in the graph spectrum upon perturbation of the graph, and its implication on determining the number of node clusters in the graph. We further demonstrate the superior empirical performance of our method in learning the number of clusters compared to that of the Lanczos algorithm.
\end{itemize}

\section{Preliminaries} \label{sec: preliminaries}

% \subsection{Graph notation} \label{subsec: graph_notation}
Graphs are the mathematical structure underpinning the formulation of networks. Let $G = (V,E)$ be an undirected graph with vertex set $V = \{v_{i}\}_{i=1}^n$. Each edge between two vertices $v_{i}$ and $v_{j}$ carries a non-negative weight $w_{ij}>0$. $w_{ij}=0$ corresponds to two disconnected nodes. For un-weighted graphs we set $w_{ij}=1$ for two connected nodes. The \emph{adjacency matrix} is defined as $\mathbf{W}$ and  $w_{ij}=[\mathbf{W} ]_{ij}$. The degree of a vertex $v_{i} \in V$ is defined as 
\begin{equation} 
    d_{i} = \sum_{j=1}^{n}w_{ij}. 
\end{equation}
The \emph{degree matrix} $\mathbf{D}$ is defined as a diagonal matrix that contains the degrees of the vertices along diagonal, i.e., $\mathbf{D}_{ii} = d_{i}$. The \emph{unnormalised graph Laplacian} matrix is defined as
\begin{equation}
    \mathbf{L} = \mathbf{D} - \mathbf{W}.
\end{equation}
As $G$ is undirected, $w_{ij}=w_{ji}$, which means that the weight matrix is symmetric and hence $\mathbf{W}$ is symmetric and given $\mathbf{D}$ is symmetric, the unnormalized Laplacian is also symmetric. As symmetric matrices are special cases of normal matrices, they are Hermitian matrices and have real eigenvalues. Another common characterisation of the Laplacian matrix is the \emph{normalised Laplacian} \citep{chung1997spectral},
\begin{equation} \label{lnorm}
    \mathbf{L}_{\mathrm{norm}} = \mathbf{D}^{-\frac{1}{2}}\mathbf{L}\mathbf{D}^{-\frac{1}{2}} = \mathbf{I} - \mathbf{W}_{\mathrm{norm}} = \mathbf{I} - \mathbf{D}^{-\frac{1}{2}}\mathbf{W}\mathbf{D}^{-\frac{1}{2}},
\end{equation}
where $\mathbf{W}_{\mathrm{norm}}$ is known as the normalised adjacency matrix \footnote{Strictly speaking, the second equality only holds for graphs without isolated vertices.}. The spectrum of the graph is defined as the density of the eigenvalues of the given adjacency, Laplacian or normalised Laplacian matrices corresponding to the graph. Unless otherwise specified, we will consider the spectrum of the normalised Laplacian.

\section{Motivations for A New Approach on Approximating and Comparing the Spectra of Large Graphs}\label{subsec:motivations}

% In this section, without loss of generality, we motivate the need for a better approach for spectral density approximation in the context of the problem of comparing large graphs.

% \subsection{Graph comparison using iterative algorithms}
For large sparse graphs with millions, or billions, of nodes, learning the exact spectrum using eigen-decomposition is unfeasible due to the $\mathcal{O}(n^{3})$ cost.
Powerful iterative methods, such as the Lanczos algorithm, which only require matrix-vector multiplications, and hence have a computational cost scaling with the number of non-zero nodes in the graph, are often used. These approaches approximate the graph spectrum with a sum of weighted Dirac delta functions, closely matching the first $m$ moments (where $m$ is the number of iterative steps used, as detailed in Appendix \ref{subsec: lanczos}) of the spectral density \citep{Ubaru2016} i.e.:
\begin{equation}
\label{eq:dirac}
    p(\lambda) = \frac{1}{n}\sum_{i=1}^{n}\delta(\lambda-\lambda_{i}) \approx \sum_{i=1}^{m}w_{i}\delta(\lambda-\lambda_{i}),
\end{equation}
where $\sum_{i=1}^m w_{i} = 1$, and $\lambda_{i}$ denotes the $i$-th eigenvalue in the spectrum.
However, such an approximation is undesirable because natural divergence measures between densities, such as the information-based relative entropy $\mathcal{D}_{\mathrm{KL}}(p||q) \in (0,\infty)$ \citep{cover2012elements},\citep{amari2007methods} as equation \ref{eq:kldivergence},
\begin{equation}
\label{eq:kldivergence}
    \mathcal{D}_{\mathrm{KL}}(p||q) =
    % \int_{\lambda \in \mathcal{D}}
    \int p(\lambda)\log \frac{p(\lambda)}{q(\lambda)} d\lambda,
\end{equation}
can be infinite for densities that are mutually singular. The use of the Jensen-Shannon divergence simply re-scales the divergence into $\mathcal{D}_{\mathrm{JS}}(p||q) \in (0,1)$. This can lead to counter-intuitive scenarios, such as an infinite (or maximal) divergence upon the removal or addition of a single edge or node in a large network, an infinite (or maximal) divergence between two graphs generated using the same random graph model and identical hyper-parameters. \newline
% This does not comply with our notion of network similarity. Two networks generated from the same stochastic process with the same hyper-parameters are highly similar and hence should have a low divergence. Similarly, the removal of an edge in a large network, such as for example two people un-friending each other on a large social network, would not in general be considered a fundamental change in the network structure. One way to circumvent the above problem is to use kernel smoothing. However, as we argue in the following, this results in losing the original moment information.

\subsection{The argument against kernel smoothing:}\label{subsec:kernelsmoothingbad}
To alleviate these limitations, practitioners typically generate a smoothed spectral density by convolving the Dirac mixture with a smooth kernel $k_\sigma(\lambda-\lambda_{i})$ \citep{takahashi2012discriminating,banerjee2008spectrum}, often a Gaussian or Cauchy \citep{banerjee2008spectrum} to facilitate visualisation and comparison. The smoothed spectral density, with reference to Equation \eqref{eq:dirac}, thus takes the form:
\begin{equation} 
    \tilde{p}(\lambda) = \int k_\sigma(\lambda-\lambda') p(\lambda') d\lambda'  =
    \int k_\sigma(\lambda-\lambda')\sum_{i=1}^{n}w_{i}\delta(\lambda'-\lambda_{i})d\lambda'  = \sum_{i=1}^{n}w_{i}k_\sigma(\lambda-\lambda_{i}). \label{eq:smootheddensity}
\end{equation}
We make some assumptions regarding the nature of the kernel function, $k_\sigma(\lambda-\lambda_{i})$, in order to prove our main theoretical result about the effect of kernel smoothing on the moments of the underlying spectral density. Both of our assumptions are met by (the commonly employed) Gaussian kernel.

%However, this introduces hyperparameters, such as the choice of convolving kernel, the smoothing parameter or the number of bins for the histogram, which heavily affect the resolution of the spectra. 
\begin{assumption}
\label{ass:infdomain}
The kernel function $k_\sigma(\lambda-\lambda_{i})$ is supported on the real line $[-\infty,\infty]$.
\end{assumption}
\begin{assumption}
\label{ass:sym}
The kernel function $k_\sigma(\lambda-\lambda_{i})$ is symmetric and permits all moments.
\end{assumption}
\begin{theorem}
The $m$-th moment of a Dirac mixture $\sum_{i=1}^{n}w_{i}\delta(\lambda-\lambda_{i})$, which is smoothed by a kernel $k_{\sigma}$ satisfying assumptions \ref{ass:infdomain} and \& \ref{ass:sym}, is perturbed from its unsmoothed counterpart by an amount $\sum_{i=1}^{n}w_{i} \sum_{j=1}^{r/2} {r \choose 2j}\mathbb{E}_{k_\sigma(\lambda)}(\lambda^{2j})\lambda_{i}^{m-2j}$, where $r=m$ if $m$ is even and $m-1$ otherwise. $\mathbb{E}_{k_\sigma(\lambda)}(\lambda^{2j})$ denotes the $2j$-th central moment of the kernel function $k_\sigma(\lambda)$.
\end{theorem}
\begin{proof}
The moments of the Dirac mixture are given as, 
\begin{equation}
    \langle \lambda^{m} \rangle = \sum_{i=1}^{n}w_{i}\int \delta(\lambda-\lambda_{i})\lambda^{m}d\lambda = \sum_{i=1}^{n}w_{i}\lambda_{i}^{m}.
\end{equation}
The moments of the modified smooth function (Equation \eqref{eq:smootheddensity}) are
\begin{equation}
    \begin{aligned}
    \langle \tilde{\lambda}^{m} \rangle & =\sum_{i=1}^nw_{i}\int k_\sigma(\lambda-\lambda_{i})\lambda^{m}d\lambda = \sum_{i=1}^nw_{i}\int k_\sigma(\lambda')(\lambda'+\lambda_{i})^{m}d\lambda' \\ & = \langle \lambda^{m} \rangle + \sum_{i=1}^{n}w_{i} \sum_{j=1}^{r/2} {r \choose 2j}\mathbb{E}_{k_\sigma(\lambda)}(\lambda^{2j})\lambda_{i}^{m-2j}.
    \end{aligned}
\end{equation}
We have used the binomial expansion and the fact that the infinite domain is invariant under shift reparametarization and the odd moments of a symmetric distribution are $0$. 
\end{proof}

\begin{remark}
The above proves that kernel smoothing alters moment information, and that this process becomes more pronounced for higher moments. Furthermore, given that $w_{i} > 0$, $\mathbb{E}_{k_\sigma(\lambda)}(\lambda^{2j}) > 0$ and (for the normalised Laplacian) $lambda_{i} > 0$, the corrective term is manifestly positive, so the smoothed moment estimates are biased.
\end{remark}

\begin{remark}
For large random graphs, the moments of a generated instance converge to those averaged over many instances \citep{feier2012methods}, hence by biasing our moment information we limit our ability to learn about the underlying stochastic process. We include a detailed discussion regarding the relationship between the moments of the graph and the underlying stochastic process in Appendix Section \ref{sec:momentsmatter}.
\end{remark}
% For simplicity we assume that the kernel function is symmetric, defined on the real line, and permits all moments, which are satisfied for the commonly used Gaussian kernel. Consider the moments of the modified smooth function in
% where the sum over $j$ is up to $r/2$, with $r$ being $m$ if $m$ is even and $m-1$ if $m$ is odd. Further, $\mathbb{E}_{k_\sigma(\lambda)}(\lambda^{2j})$ denotes the $2j$-th central moment of the kernel function $k_\sigma(\lambda)$. 
\section{An Information Theoretically Optimal Approach to the Problem of Smooth Spectra for Massive Graphs}
For large, sparse graphs corresponding to real networks with millions or billions of nodes, where eigen-decomposition is intractable, we may still be able to compute a certain number of matrix-vector products, which we can use to get unbiased estimates of the spectral density moments, using stochastic trace estimation (as explained in Appendix \ref{subset: stoctrace}). We can settle on a unique spectral density which satisfies the given moment information exactly, known as the density of Maximum Entropy explained in Section \ref{subsec: maxent_method}. 

\subsection{Maximum Entropy: MaxEnt}
\label{subsec: maxent_method}

The method of maximum entropy, hereafter referred to as \emph{MaxEnt} \citep{Presse2013}, is information-theoretically optimal in so far as it makes the least additional assumptions about the underlying density \citep{inftheoryjaynes} and is flattest in terms of the KL divergence compared to the uniform \citep{granziol2019meme}. %the most non-committal with regard to missing information \citep{Jaynes1957}. 
%Intuitively, on a bounded domain, the most conservative distribution, the distribution of maximum entropy, is the one that assigns equal probability to all the accessible states. Hence, the method of maximum entropy can be thought of choosing the flattest, or most equiprobable distribution, satisfying the given constraints.
To determine the spectral density $p(\lambda)$ using MaxEnt, we maximise the entropic functional
\begin{equation} \label{BSG}
    S = - \int p(\lambda)\log p(\lambda)d\lambda- \sum_{i}\alpha_{i}\bigg[\int p(\lambda)\lambda^{i}d\lambda - \mu_{i}\bigg],
\end{equation}
with respect to $p(\lambda)$, where $ \mathbb{E}_{p}[ \lambda^{i}]= \mu_{i}$ are the power moment constraints on the spectral density, which are estimated using stochastic trace estimation (STE) as explained in Appendix \ref{subset: stoctrace}. The resultant entropic spectral density has the form
\begin{equation}\label{eq:maxent_distribution}
    p(\lambda \vert \{ \alpha_{i} \}) = \exp[-(1+\sum_{i}\alpha_{i} \lambda^{i})],
\end{equation}
where the coefficients $\{ \alpha_{i} \}_{i=1}^{m}$ are derived from optimising \eqref{BSG}. We use the MaxEnt algorithm, proposed in \citep{granziol2019meme} to learn these coefficients. For simplicity, we denote $p(\lambda \vert \{ \alpha_{i}\}_{i=1}^{m})$ as  $p(\lambda )$. \footnote{We make our Python code available on https://github.com/diegogranziol/Python-MaxEnt}

\begin{wrapfigure}[10]{L}{0.62\textwidth}
\vspace{-14pt}
    \begin{minipage}{0.62\textwidth}
\begin{algorithm}[H]
 	\caption{Entropic Graph Spectrum(EGS) Learner.}
 	\label{alg:EGS}
 	\begin{algorithmic}[1]
 		\STATE {\bfseries Input:} Normalized Laplacian $\mvec{L}_{\mathrm{norm}}$, number of probe vectors $d$, number of moments used $m$
 		\STATE {\bfseries Output:} EGS $ p(\lambda)$
 		\STATE Moments $\{ \mu_i \}_{i=1}^m$ $\leftarrow$ STE $\left(\mvec{L}_{\mathrm{norm}}, d, m \right)$
 		\STATE MaxEnt Coefficients \\ $\{ \alpha_i \}_{i=1}^m$ $\leftarrow$ MaxEnt Algorithm $\left(\{ \mu_i \}_{i=1}^m \right)$
 		\STATE Entropic graph spectrum $p(\lambda) = \exp[-(1+\sum_{i}\alpha_{i}\lambda^{i})]$
 	\end{algorithmic}
\end{algorithm}
\end{minipage}
\end{wrapfigure}

\subsection{The Entropic Graph Spectral Learning algorithm}

The full algorithm for learning the entropic graph spectrum (EGS) is summarized in Algorithm \ref{alg:EGS}. We first estimate the $m$ moments of the normalised graph Laplacian $\{ \mu_i \}_{i=1}^m$ via STE, then use the moment information to solve for MaxEnt coefficients $\{ \alpha_i \}_{i=1}^m$ and compute the EGS via Equation \eqref{eq:maxent_distribution}.

\section{Visualising the Modelling Power of EGS}\label{sec:modelling_power_of_maxent}
Having developed a theory as to why a smooth, exact moment matched approximation of the spectral density is crucial to learning the characteristics of the underlying stochastic process, and having proposed a method (Algorithms \ref{alg:EGS}) to learn such a density, we test the practical utility of our method and algorithm on examples where the limiting spectral density is known.

\subsection{Erd\H{o}s-R\'{e}nyi graphs and The semi-circle law} \label{subsec:semicircle_and_beyond}
For Erd\H{o}s-R\'{e}nyi graphs with edge creation probability $p \in (0,1)$, and $np \rightarrow \infty$, the limiting spectral density of the normalised Laplacian converges to the semi-circle law and its Laplacian converges to the free convolution of the semi-circle law and $\mathcal{N}(0,1)$ \citep{jiang2012empirical}. We consider here to what extent our EGS learnt with finite moments can effectively approximate the density. Wigner's density is fully defined by its infinite number of central moments given by $\mathbb{E}_{\mu}(\lambda^{2n}) = (R/2)^{2n}C_{n}$, where $C_{n}\times (n+1) = {2n \choose n}$ are known as the Catalan numbers. As a toy example we generate a semi-circle centered at $\lambda=0.5$ with $R=0.5$ and use the analytical moments to compute its corresponding EGS (FIG \ref{fig:semi_circle_and_maxent}). As can be seen in FIG \ref{fig:semicirclepure}, for $m=5$ moments, the central portion of the density is already well approximated, but the end points are not. This is largely corrected for $m=30$ moments. 
\begin{figure}[t]
	\begin{subfigure}{0.5\linewidth}
		\centering
	\centering
    \includegraphics[trim=0.7cm 0.1cm 1.0cm 0.0cm, clip, width=1.0\linewidth]{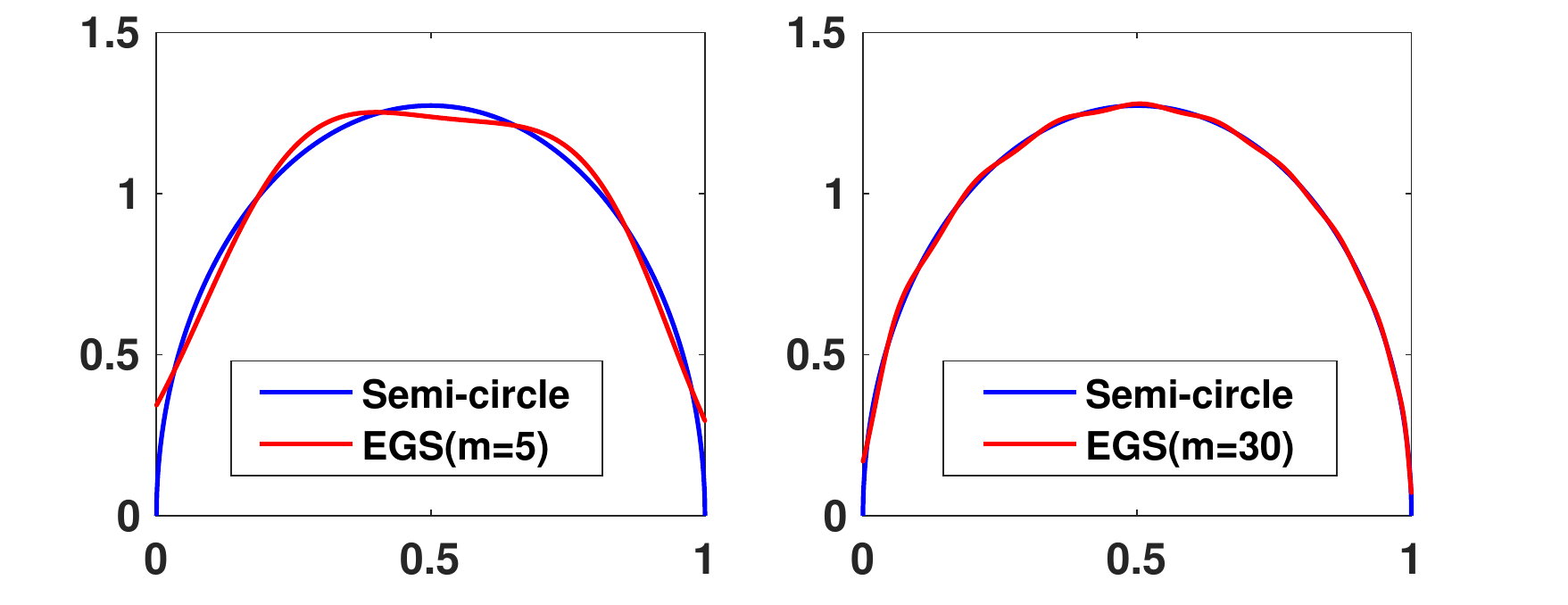}
	\caption{}
	\label{fig:semicirclepure}
	\end{subfigure}\hspace{10pt}
	\begin{subfigure}{0.45\linewidth}
		\centering
		% trim={<left> <lower> <right> <upper>}
	\includegraphics[trim=0.1cm 0cm 0.0cm 0cm, clip, width=1.0\linewidth]{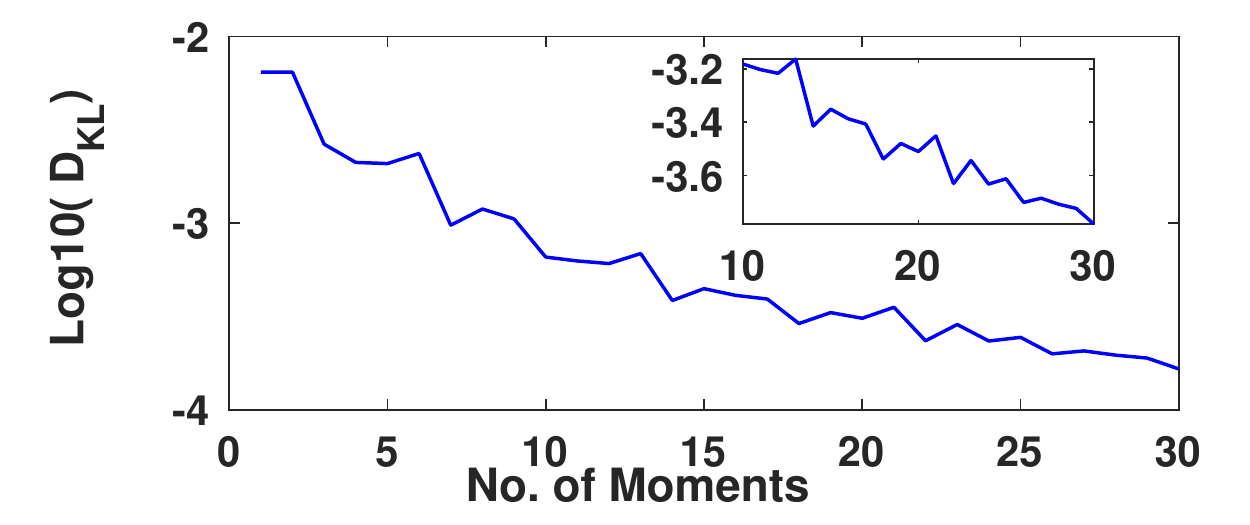}
	\caption{} \label{fig:maxentsemicirclekl}
	\end{subfigure}
	\label{kl_between_semi_circle_and_maxent}
	\caption{EGS fit to a semi-circle density that is centered at 0.5 and has a radius of 0.5 $[x_{0},R] = [0.5,0.5]$ for different moment number $m$. (a) visualises the quality of fit for $m=5$ and $m=30$. (b) shows the KL divergence between the true semi-circle density and the EGS.}    \label{fig:semi_circle_and_maxent}
\vspace{-12pt}
\end{figure}

We generate an Erd\H{o}s-R\'{e}nyi graph with $p=0.001$ and $n=5000$, and learn the moments using stochastic trace estimation. We then compare the fit between the EGS computed using a different numbers of input moments $m = 3, 30, 60,100$ and the graph eigenvalue histogram computed by eigen-decomposition. We plot the results in FIG \ref{fig:maxenter}. One striking difference between this experiment and the previous one is the number of moments needed to give a good fit. This can be seen especially clearly in the top left subplot of FIG \ref{fig:maxenter}, where the 3 moment, i.e Gaussian approximation, completely fails to capture the bounded support of the spectral density. Given that the exponential polynomial density is positive everywhere, it needs more moment information to learn the regions of boundedness of the spectral density in its domain. In the previous example we artificially alleviated this phenomenon by putting the support of the semi-circle within the entire domain. It can be clearly seen that increasing moment information successively improves the fit to the support FIG \ref{fig:maxenter}. Furthermore, the magnitude of the oscillations, which are characteristic of an exponential polynomial function, decay in magnitude for larger moments. 

\begin{table}[t]
\vspace{-12pt}
\begin{tabular}{lc}
\begin{minipage}{.53\textwidth}
\begin{figure}[H]
	\centering
   	 \includegraphics[trim=1.3cm 0.7cm 1.5cm 0cm, clip, width=1.0\linewidth]{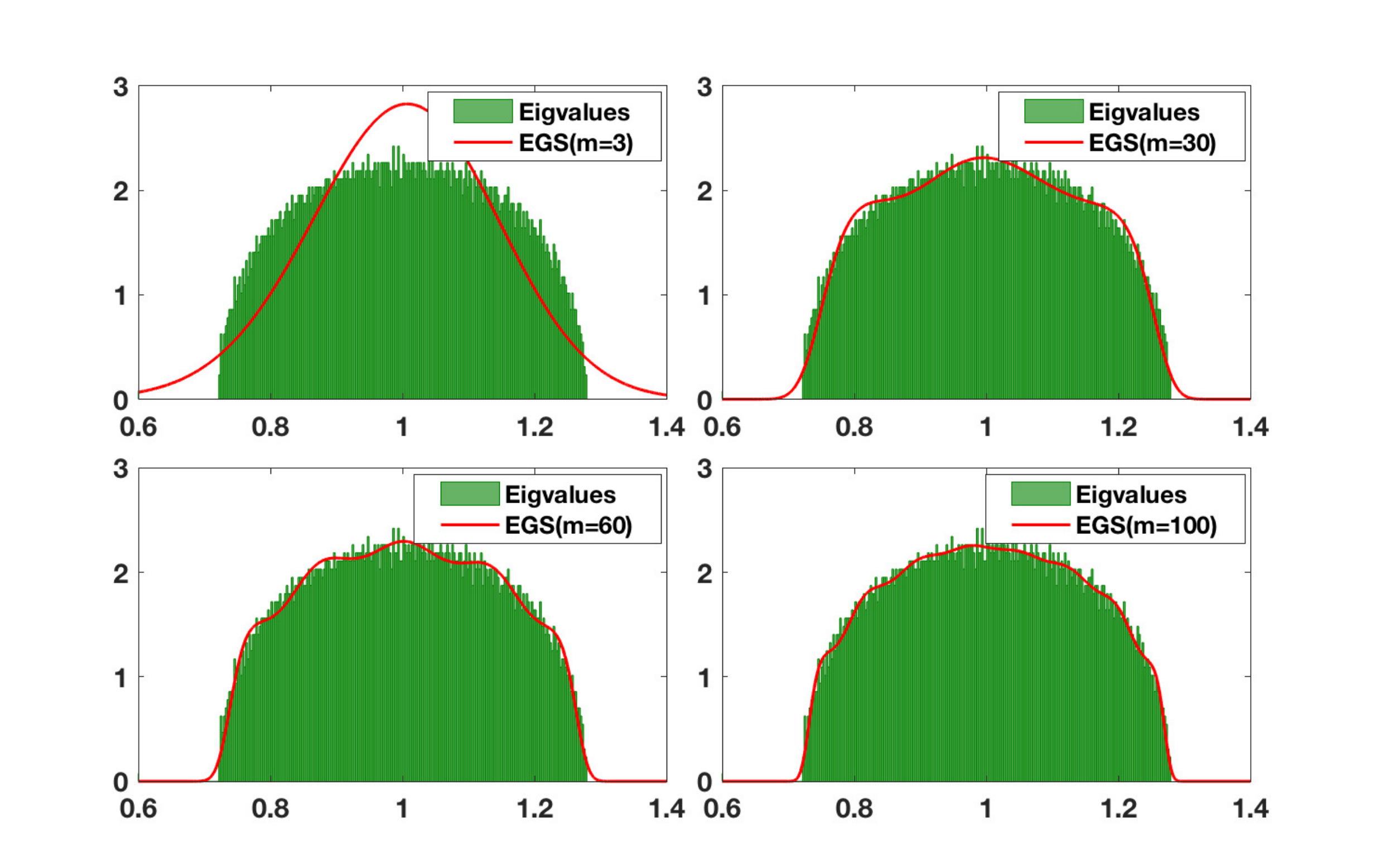}
	\caption{EGS fit to randomly generated $p=0.001, n = 5000$ Erd\H{o}s-R\'{e}nyi graph. The number of moments $m$ used increases from 3 to 100 and the number of bins used for the eigenvalue histogram is $n_{b}=500$.}\label{fig:maxenter}
\end{figure}
\end{minipage}\hspace{10pt}
\begin{minipage}{.4\textwidth}
\begin{figure}[H]
	\centering
	% trim={<left> <lower> <right> <upper>}
    \includegraphics[trim=1.3cm 0.0cm 1.5cm 0.5cm, clip, width=1.0\linewidth]{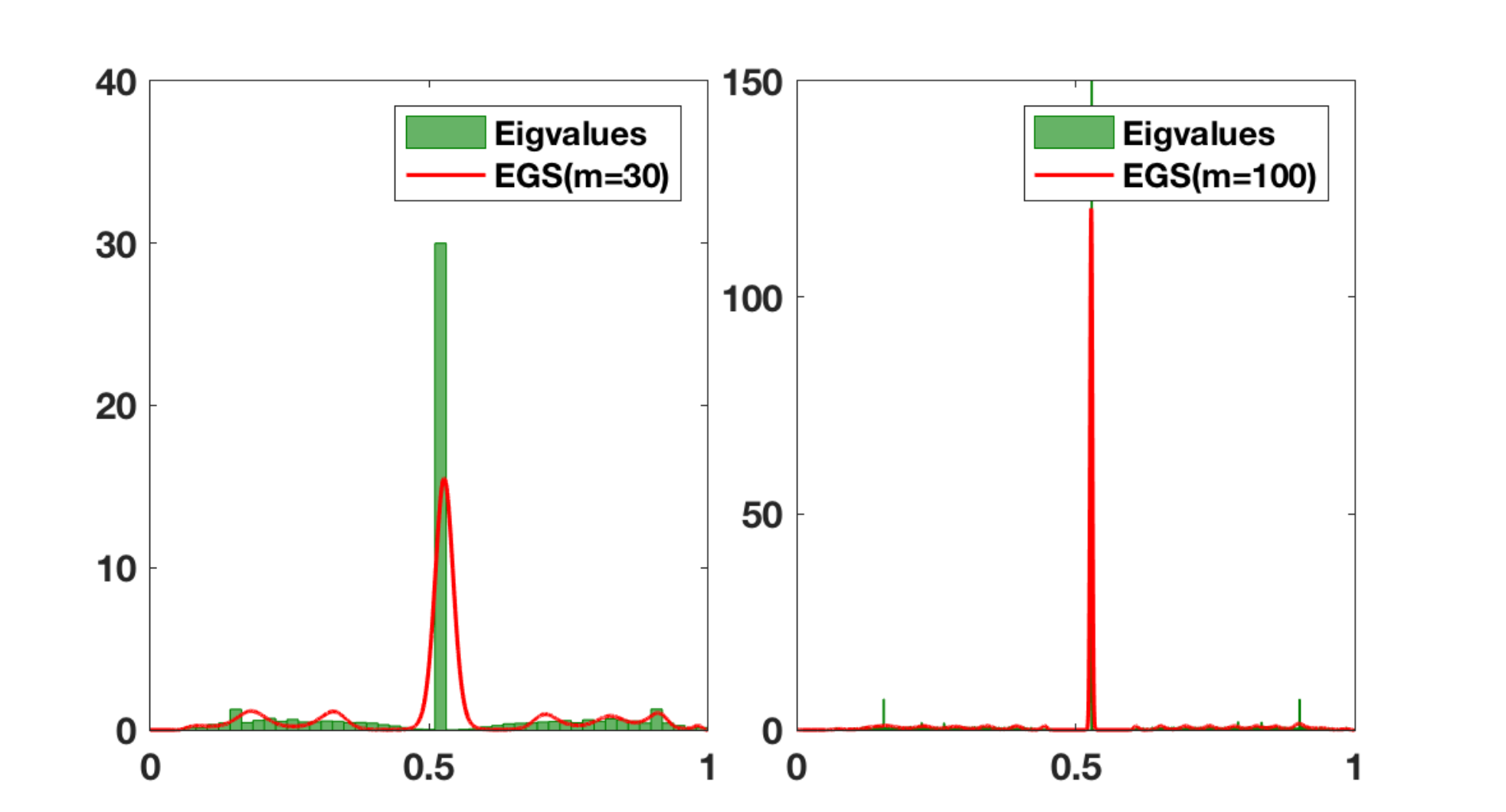}
    \caption{EGS fit to randomly generated $n = 5000$ Barab\'{a}si-Albert graph. The number of moments used for computing EGSs and the number of bins used for the eigenvalue histogram are $m=30$, $n_{b}=50$ (Left) and $m=100$, $n_{b}=500$ (Right).}\label{fig:maxentba}
\end{figure}
\end{minipage}
\end{tabular}
\vspace{-12pt}
\end{table}	

% \begin{figure}[tbp]	
% 	\centering
% 	 \includegraphics[trim=1.3cm 0.7cm 1.5cm 0cm, clip, width=1.0\linewidth]{Figures/er_hist_maxent_combined_nb500-eps-converted-to.pdf}
% 	\caption{Maximum Entropy distribution fit to randomly generated $p=0.001, n = 5000$ Erd\H{o}s-R\'{e}nyi graph. The number of moments used for computing Maximum Entropy distributions increases from $m=3$ to $m=100$ abd the number of bins used for the eigenvalue histogram is $n_{b}=500$.}\label{fig:maxenter}
% \end{figure} 

\subsection{Beyond the semi-circle law} \label{subsec:ba_networks}

For the adjacency matrix of an Erd\H{o}s-R\'{e}nyi graph with $p \propto 1/n$, the limiting spectral density does not converge to the semi-circle law and has an elevated central portion, and the scale free limiting density converges to a triangle like distribution \citep{farkas2001spectra}. For other random graph, such as the Barab\'{a}si-Albert \citep{barabasi1999emergence}, also known as the scale-free network, the probability of a new node being connected to a certain existing node is proportional to the number of links that existing node already has, violating the independence assumption required to derive the semi-circle density. We plot a Barab\'{a}si-Albert network ($n=5000$) and, similar to Section \ref{subsec:semicircle_and_beyond}, we learn the EGS and plot the resulting spectral density against the eigenvalue histogram,  shown in FIG \ref{fig:maxentba}. For the Barab\'{a}si-Albert network, due to the extremity of the central peak, a much larger number of moments is required to get a reasonable fit. We also note that increasing the number of moments is  akin to increasing the number of bins in terms of spectral resolution, as seen in FIG \ref{fig:maxentba}.

% \begin{figure}[t]	
% 	\centering
% 	\includegraphics[trim=1.3cm 0.0cm 1.7cm 0.5cm, clip, width=1.0\linewidth]{Figures/ba_hist_maxent_combined_nb50_500-eps-converted-to.pdf}
% 	\caption{Maximum Entropy distribution fit to randomly generated $n = 5000$ Barab\'{a}si-Albert graph. The number of moments used for computing Maximum Entropy distributions and the number of bins used for the eigenvalue histogram are $m=30$, $n_{b}=50$ (Left) and $m=100$, $n_{b}=500$ (Right).}\label{fig:maxentba}
% \end{figure}

% \section{Experiments}\label{sec:all_experimetns}
\section{EGS for Measuring Graph Similarity}\label{sec:all_experimetns}

% We perform experiments to examine the performance of our MaxEnt-based method on two applications. Firstly, we recover the parameters of random graph models as well as classify the network type by measuring the symmetric KL divergence among various graphs. Secondly, we learn the number of clusters in a network by performing gradient descent and quadrature on its surrogate MaxEnt spectral density.
In this section, we test the use of our EGS in combination with symmetric KL divergence to measure similarity between different types of synthetic and real world graphs. Note that our proposed EGS,  based on the MaxEnt distribution, enables the symmetric KL divergence to be computed analytically - this we show in Appendix \ref{sec: analytic_diff_ent_and_divergence}. We first investigate the feasibility of recovering the parameters of random graph models, and then move onto classifying the network type as well as computing graph similarity among various synthetic and real world graphs.

\subsection{Inferring parameters of random graph models}

We investigate whether one can recover the network parameter values of a graph via its learned EGS. We generate a random graph of a given size and parameter value (e.g., $n=50, p=0.6$) and learn its entropic spectral characterisation using our EGS learner (Algorithm \ref{alg:EGS}). Then, we generate another graph of the same size but learn its parameter value by minimising the symmetric-KL divergence between its entropic spectral surrogate and that of the original graph. We repeat the above procedures for different random graph models i.e. Erd\H{o}s-R\'{e}nyi (ER), Watts-Strogatz (WS) and Barab\'{a}si-Albert (BA) and different graph sizes ($n=50, 100, 150$), and the results are shown in Table \ref{table:learn_synnet_para}. It can be seen that,  given  the approximate EGS, we are able to learn  well the parameters of the graph producing that spectrum.

\begin{table}[t]
	\vspace{-12pt}
\begin{tabular}{cc}
\begin{minipage}{0.58\textwidth}
\begin{table}[H]
	\caption{Average parameters estimated by our MaxEnt-based method for the 3 types of network. The number of nodes in the network is denoted by $n$.}\label{table:learn_synnet_para}
    \centering
	\begin{tabular}{lcccr}
		\toprule
		$n$ & 50  & 100  & 150 \\ 
		%& $($p=0.6$)$ & $($p=0.4$)$ & ($m=0.4n$)\\
		\midrule
		ER ($p=0.6$)     & $0.600$   & $0.598$  & $ 0.604$   \\
		WS ($p=0.4$)  & $0.468$   & $0.454$   & $0.414$   \\
		BA ($r=0.4n$)   & $18.936$ & $40.239$  & $58.428$   \\
		\bottomrule
	\end{tabular}
\end{table}
\end{minipage}&
\begin{minipage}{.4\textwidth}
 \begin{table}[H]
	\caption{Minimum KL divergence between the EGSs of random networks and that of a large BA graph and YouTube network.}
	\label{table:learn_synnet_real_type}
	\centering
	\begin{tabular}{lcccr}
		\toprule
		& Large BA  & YouTube \\ 
		\midrule
		ER  & $2.662$ &$7.728$\\
		WS & $7.612$ &$ 9.735$ \\
		BA & $\mathbf{2.001}$ & $ \mathbf{7.593} $\\
		\bottomrule
	\end{tabular}
\end{table} 
\end{minipage}
\end{tabular}
\vspace{-12pt}
\end{table}

% ------------------------------------------------------------------% 
%                   Experiments
% ------------------------------------------------------------------%

\subsection{Learning real world network types}\label{subsec: network_classification_maxent_kl}

%We also test our method for robustness against changing spectral properties with network scale, by running an experiment where we generate a synthetic graph of $n=5000$ with given paramaters and then try to minimize the divergence between random graphs of free parameters of fixed nodal size $1000$, where we successfully recover the parameters of the scale free network. 

Determining which random graph model best fit a real-world network, characterised by their spectral divergence can lead to better understanding of its dynamics and characteristics. This has been explored for small biological networks \citep{takahashi2012discriminating} where full eigen-decomposition is viable. Here, we conduct similar experiments for large networks based on our EGS method. We first test on a large ($5000$-node) synthetic BA network. By minimising the symmetric KL divergence between its EGS and those of small (1000-node) random networks (ER, WS, BA), we successfully recover its own type. As a real-world use case, we further repeat the experiment to determine which random network can best model the YouTube network from the SNAP dataset \citep{snapnets} and find, as shown in Table \ref{table:learn_synnet_real_type}, that the BA gives the lowest divergence, which aligns with other findings for social networks \citep{barabasi1999emergence}.

% \begin{table}[htb!]
% 	\caption{Minimum KL divergence between Entropic Spectrum of Youtube and that of synthetic networks}
% 	\label{table:learn_real_para}
% 	\begin{center}
% 		\begin{small}
% 			\begin{sc}
% 				\begin{tabular}{|l|c|c|c|r|}
% 					\toprule
% 					\hline
% 					& Synthetic  & Youtube \\ \hline
% 					\midrule
% 					Erd\H{o}s-R\'{e}nyi  & $2.662$ &$7.728$\\
% 					Watts-Strogatz & $7.6123$ &$ 9.735$ \\
% 					Barab\'{a}si-Albert & $\mathbf{2.001}$ & $ \mathbf{7.593} $\\
% 					\midrule
% 					\hline
% 				\end{tabular}
% 			\end{sc}
% 		\end{small}
% 	\end{center}
% 	\vskip -0.1in
% \end{table} 

\begin{wrapfigure}{L}{0.6\textwidth}
    \begin{minipage}{0.6\textwidth}
\begin{figure}[H]
\vspace{-12pt}
    \centering
    \includegraphics[width = 1.0\linewidth]{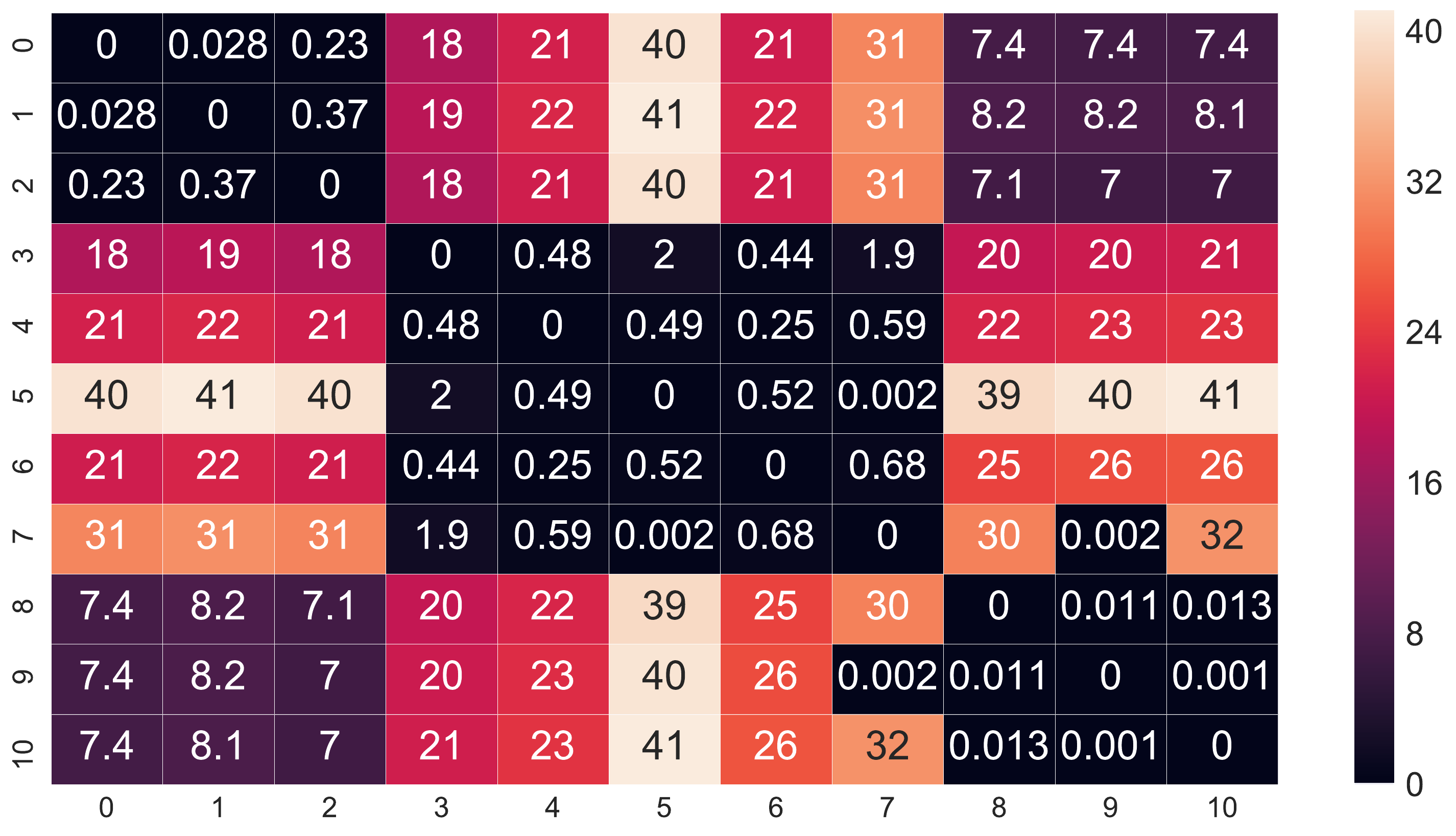}
    \caption{Symmetric KL heatmap between 9 graphs from the SNAP dataset: (0) bio-human-gene1,
(1) bio-human-gene2,
(2) bio-mouse-gene,
(3) ca-AstroPh,
(4) ca-CondMat,
(5) ca-GrQc,
(6) ca-HepPh,
(7) ca-HepTh,
(8) roadNet-CA,
(9) roadNet-PA,
(10) roadNet-TX.}
    \label{fig:graphsinthewild}
\end{figure}
\end{minipage}
\end{wrapfigure} 

\subsection{Comparing different real world networks}
We now consider the feasibility of comparing real world networks using EGSs. Specifically, we take $3$ biological networks, $5$ citation networks and $3$ road networks from the SNAP dataset \citep{snapnets}, and compute the symmetric KL divergences between their EGS with $m=100$ moments. We present the results in a heat map (FIG \ref{fig:graphsinthewild}). We see very clearly that the intra-class divergences between the biological, citation and road networks are much smaller than their inter-class divergences. This strongly suggests that the combination of our EGS method and the symmetric KL divergence can be used to identify similarity in networks. Furthermore, as can be seen in the divergence between the human and mouse network, the spectra of human genes are more closely aligned with each other than they are with the spectra of mouse genes. This suggests a reasonable amount of intra-class distinguishability as well.

% \section{Clustering Using the Eigenspectra}\label{subset:clustering_with_spec}
\section{EGS for Estimating Cluster Number}\label{subset:clustering_with_spec}
It is known from spectral graph theory \citep{chung1997spectral}, that the number multiplicity of the $0$ eigenvalue in the Laplacian (and the normalized Laplacian) is equal to the number of connected components in the graph \citep{von2007tutorial}. Previous literature has argued \citep{ubaruapplications}, that for a small amount of inter-cluster connections by matrix perturbation theory \citep{bhatia2013matrix} we should expect a number of eigenvalues close to $0$, we make this argument precise with the following Theorem \ref{theorem:normlaplaceproof}.

\begin{theorem}
\label{theorem:normlaplaceproof}
The normalised Laplacian eigenvalue, perturbated by adding a single edge between nodes $1$ and $m+1$ from two previously disconnected clusters $A$ and ,$B$ is bounded to first order by
\begin{equation}
    \bigg|\frac{1}{\sqrt{d_{1}d_{a}}}+\frac{1}{\sqrt{d_{m+1}d_{b}}}-\frac{2}{\sqrt{d_{1}d_{m+1}}}\bigg|,
\end{equation}
where $d_{i}$ denotes the degree of node $i$ and $1/\sqrt{d_{a}} = \sum_{g \in (g,1)}1/\sqrt{d_{g}}$ and similarly $1/\sqrt{d_{b}} = \sum_{g \in (g,m+1)}1/\sqrt{d_{g}}$, where $\sum_{g \in (g,1)}$ denotes the sum over all nodes connecting to node $1$.
\end{theorem}
\begin{proof}
Using Weyl's bound on Hermination matrices \citep{bhatia2013matrix},
\begin{equation}
\begin{aligned}
    %& |H-G| \geq |\lambda'_{i}-\lambda_{i}| \\
    %& \therefore \xded{\Delta \lambda_{i}} \leq |H-G|
    \Delta \lambda_{i}=|\lambda'_{i}-\lambda_{i}| \leq ||\mathbf{L}_{G'}-\mathbf{L}_G||.
\end{aligned}
\end{equation}
By the definition of the normalized Laplacian $\tilde{L}_{i,j} = L_{i,j}/\sqrt{d_{i}d_{j}}$
\begin{equation}
    \Delta \tilde{L} = \sum_{g \in (g,1)}\bigg(\frac{1}{\sqrt{d_{1}d_{g}}}-\frac{1}{\sqrt{(d_{1}+1)d_{g}}}\bigg) + \sum_{g \in (g,m+1)}\bigg(\frac{1}{\sqrt{d_{m+1}d_{g}}}-\frac{1}{\sqrt{(d_{m+1}+1)d_{g}}}\bigg) - \frac{2}{\sqrt{d_{1}d_{m+1}}},
\end{equation}
to first order in the binomial expansion. We hence have the result.
\end{proof}
\begin{corollary}
For two clusters with identical degree $d \gg 1$, connected by a single inter-cluster link, the zero eigenvalue eigenvalue is perturbed to first order by at most $ \Delta \lambda_{0} = \frac{1}{d} $.
\end{corollary}
\begin{remark}
Hence for $E$ inter-cluster connections, our bound goes as $E/d$ and hence the intuition of a small change in the $0$ eigenvalue holds if the number of edges between clusters is much smaller than the degree of the nodes within the clusters.
\end{remark}

\begin{table}[hbt]
\vspace{-12pt}
\begin{tabular}{lc}
\begin{minipage}{.48\textwidth}
\begin{algorithm}[H]
	\caption{Cluster Number Estimation.}
	\label{alg:clusteralg}
	\begin{algorithmic}[1]
		\STATE {\bfseries Input:} Normalized graph Laplacian $\mvec{L}_{\mathrm{norm}}$, graph dimension $n$, tolerance $\eta$
		
		\STATE {\bfseries Output:} Number of clusters $N_{c}$
		\STATE EGS $p(\lambda)\leftarrow$ Algorithm \ref{alg:EGS}($\mvec{L}_{\mathrm{norm}}$)
		%How shall we write this bit??
		\STATE Find minimum $\lambda_*$ that satisfy $\frac{dp(\lambda)}{d\lambda}|_{\lambda=\lambda_*} \leq \eta \thinspace$ and  $\frac{d^{2}p(\lambda)}{d\lambda^{2}}_{\lambda=\lambda_*} > 0$
		\STATE Calculate $N_{c} = n\int_{0}^{\lambda*} p(\lambda)d\lambda$	
		%\UNTIL{$noChange$ is $true$}
	\end{algorithmic}
\end{algorithm}
\end{minipage}&
\begin{minipage}{.5\textwidth}
\begin{figure}[H]
	\centering
	\includegraphics[width=0.9\linewidth]{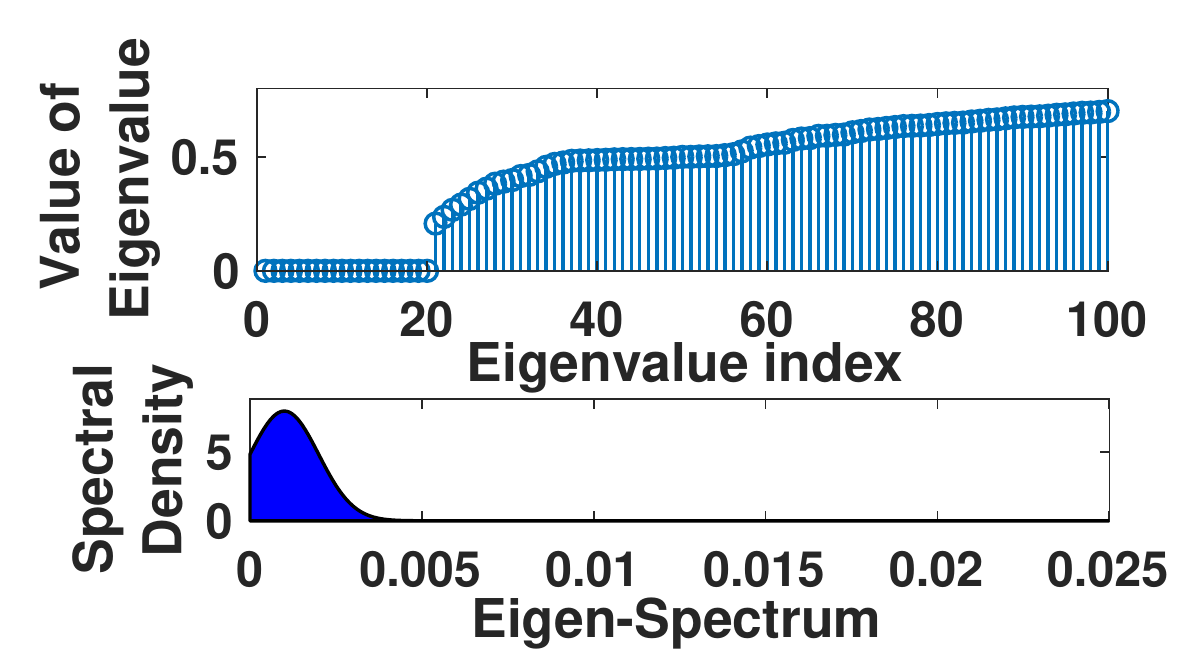}
	\caption{Eigenvalues of the Email dataset with clear spectral gap and $\lambda_* \approx 0.005$. The shaded area multiplied by the number of nodes $n$ predicts the number of clusters.}
	\label{fig:email}	
\end{figure}
\end{minipage}
\end{tabular}
\vspace{-12pt}
\end{table}	

%Whilst the idea of strong communities being nearly disconnected components, is not novel \citep{mcgraw2008Laplacian} and has been used in community detection algorithms \citep{capocci2005detecting}, the authors are unaware of the derivation of the above bounds and consider them instrumental theoretical footing for learning the number of clusters in a connected graph.

\subsection{Learning the number of clusters in large graphs}
For the case of large sparse graphs, where only iterative methods such as the Lanczos algorithm can be used, the same arguments from  Section \ref{subsec:motivations} apply. 
This is because the Dirac's delta functions are now weighted, and to obtain a reliable estimate of the eigengap, one must smooth the Dirac's delta functions. We would expect a smoothed spectral density plot to have a spike near $0$. We expect the moments of the spectral density to encode this information and the mass of this peak to be spread. We hence look for the first spectral minimum in the EGS and calculate the number of clusters as shown in Algorithm \ref{alg:clusteralg}. We conduct a set of experiments to evaluate the effectiveness of our spectral method in Algorithm \ref{alg:clusteralg} for learning the number of distinct clusters in a network, where we compare it against the Lanczos algorithm with kernel smoothing on both synthetic and real-world networks.

\subsubsection{Synthetic networks}
The synthetic data consists of disconnected sub-graphs of varying sizes and cluster numbers, to which a small number of intra-cluster edges are added. We use an identical number of matrix vector multiplications, i.e., $m=80$ (see Appendix \ref{subsec:implementation_details} for experimental details for both EGS and Lanczos methods), and estimate the number of clusters and report the fractional error. The results are shown in Table \ref{table:fractional_error_syn}. In each case, the method achieving lowest detection error is highlighted in bold. It is evident that the EGS approach outperforms Lanczos as the number of clusters and the network size increase. 
We observe a general improvement in performance for larger graphs, visible in the differences between fractional errors for EGS as the graph size increases and not kernel-smoothed Lanczos. %This is to be expected as the true spectral density

\begin{table}[h]
	\caption{Fractional error in cluster number detection for synthetic networks using EGS and Lanczos methods with 80 moments. $n_c$ denotes the number of clusters in the network and $n$ the number of nodes.}\label{table:fractional_error_syn}
    \centering
	\begin{tabular}{lcccr}
		\toprule
		$n_c$ ($n$) & 9 (270)  & 30 (900) & 90 (2700) & 240 (7200) \\ 
		\midrule
		 Lanczos & $ \mathbf{3.20 \times 10^{-3}}$  & $ 1.41\times 10^{-2}$ & $1.81\times 10^{-2}$ & $2.89\times 10^{-2}$ \\
		 EGS  & $9.70 \times 10^{-3}$  & $ \mathbf{6.40 \times 10^{-3}}$ & $\mathbf{5.80\times 10^{-3}}$ & $\mathbf{3.50\times 10^{-3}}$ \\
		\bottomrule
	\end{tabular}
\end{table}

To test the performance of our approach for networks that are too large to apply eigen-decomposition, we generate two large networks by mixing the ER, WA, BA random graph models. The first large network has a size of  201,600 nodes and comprises 305 interconnected clusters whose size varies from 500 to 1000 nodes. The second large network has a size of 404,420 nodes and comprises interconnected 1355 clusters whose size varies from 200 to 400 nodes. The results in FIG \ref{fig:largeSyntheNet} show that for both methods, the detection error generally decreases as more moments are used, and our EGS approach again outperforms the Lanczos method for both large synthetic networks.

\begin{table}[t]
\begin{tabular}{ll}
\begin{minipage}{.45\textwidth}
\begin{figure}[H]
	\centering
	\begin{subfigure}{0.5\linewidth}
		\centering
	    \includegraphics[trim=0cm 0cm 0.1cm 0.0cm, clip, width=1.0\linewidth]{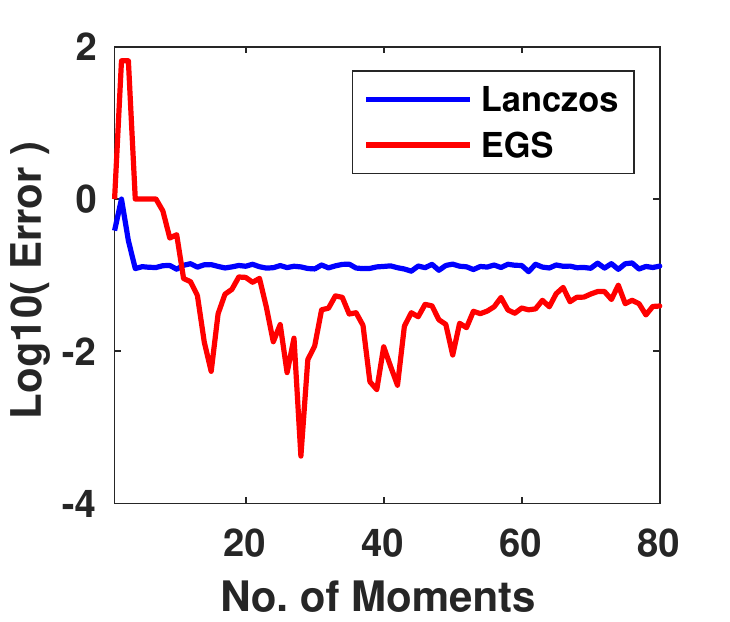}
	    \caption{305 clusters}
	    \label{subfig:emailerror1}	
	\end{subfigure}%
	\begin{subfigure}{0.5\linewidth}
		\centering
    	\includegraphics[trim=0cm 0cm 0.1cm 0.0cm, clip, width=1.0\linewidth]{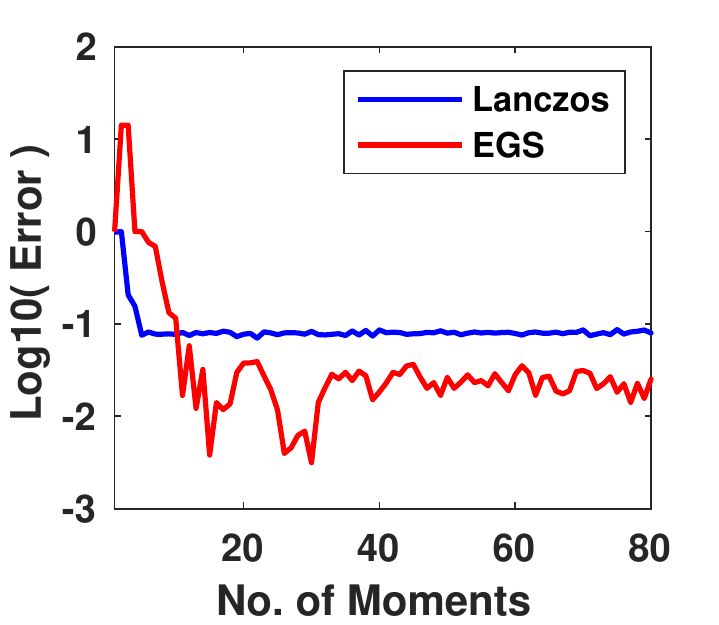}
	    \caption{1,355 clusters}
	    \label{subfig:emailerror2}	
	\end{subfigure}%
	\caption{Log error of cluster number detection using EGS and Lanczos methods on large synthetic networks with (a) 201,600 nodes and 305 clusters and (b) 404,420 nodes and 1,355 clusters.}
	\label{fig:largeSyntheNet}	
\end{figure} 
\end{minipage}&
\begin{minipage}{.45\textwidth}
\begin{figure}[H]
    \centering
	\begin{subfigure}{0.5\linewidth}
		\centering
	\includegraphics[trim=0cm 0cm 0.1cm 0.0cm, clip,width=1.0\linewidth]{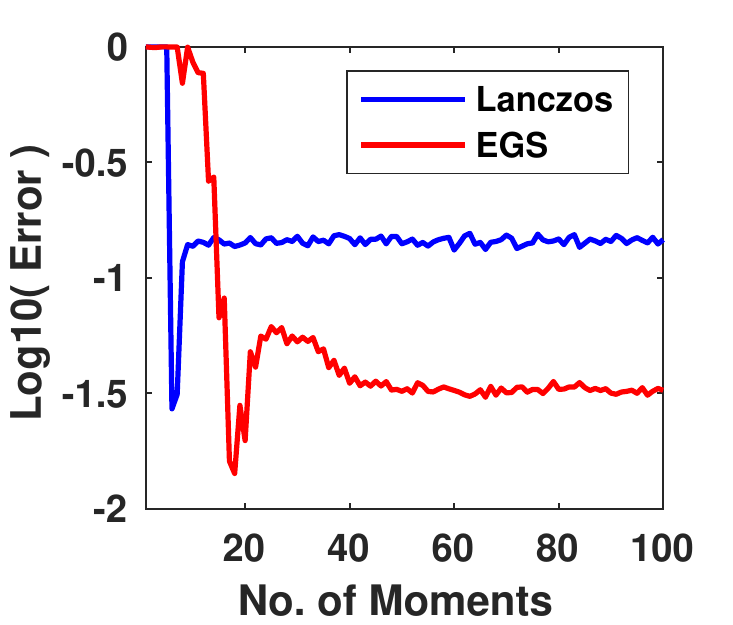}
    \caption{Email dataset}
    \label{fig:emailerror}	
	\end{subfigure}%
	\begin{subfigure}{0.5\linewidth}
		\centering
		\includegraphics[trim=0cm 0cm 0.1cm 0.0cm, clip,width=1.0\linewidth]{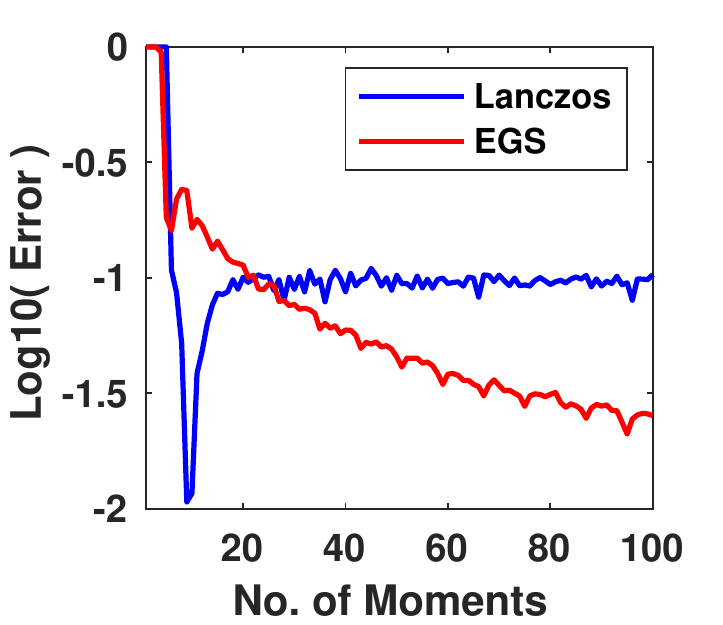}
		\caption{NetScience dataset}
		\label{fig:netscienceerror}
	\end{subfigure}
	\caption{Log error of cluster number detection using EGS and Lanczos methods on small-scale real world networks (a) Email network of 1,003 nodes and (b) NetScience network of 1,589 nodes.}
	\label{fig:test2}
\end{figure}
\end{minipage}
\end{tabular}
\vspace{-12pt}
\end{table}

\subsubsection{Small real world networks}\label{subsubsec:small_real_network}
We next experiment with relatively small real world networks, such as the Email network in the SNAP dataset, which is an undirected graph where the $n=1,003$ nodes represent members of a large European research institution and the edges represent the existence of email communication between them. For such network, we can still calculate the ground-truth number of clusters by computing the eigenvalues explicitly and finding the spectral gap near $0$. For the Email network, we count $20$ very small eigenvalues before a large jump in magnitude (measured in the log scale) and set this as the ground-truth. This is shown in FIG \ref{fig:email}, where we display the value of each of the eigenvalues in increasing order and how this results in a broadened peak in the EGS. The area under the curve multiplied by the number of network nodes is the number of clusters $n_{c}$. 
We note that the number $20$ differs from the value of $42$ given by the number of departments at the research institute in this dataset. A likely reason for this ground-truth inflation is that certain departments, Astrophysics, Theoretical Physics and Mathematics for example, may collaborate to such an extent that their division in name may not be reflected in terms of node connection structure.
% \begin{figure}[t]
% 	\centering
% 	\includegraphics[width=1\linewidth]{Figures/EmailStemGraph}
% 	\caption{Eigenvalues of the Email dataset, with clear spectral gap along with the corresponding spectral density near the origin, showing a minimum at the value of the eigengap. The shaded area multiplied by the number of nodes $n$ predicts the number of clusters.}
% 	\label{fig:email}	
% \end{figure} 
We plot the log error against the number of moments for both EGS and Lanczos in FIG \ref{fig:emailerror}, with EGS showing superior performance. We repeat the experiment on the Net Science collaboration network, which represents a co-authorship network of $1,589$ scientists ($n = 1,589$) working on network theory and experiment \citep{newman2006finding}. The results in FIG \ref{fig:netscienceerror} show that EGS quickly outperforms the Lanczos algorithm after around $20$ moments.

\subsubsection{Large real world networks} \label{subsubsec:large_real_network}

For large datasets with $n \gg 10^{4}$, where the Cholesky decomposition becomes completely prohibitive even for powerful machines, we can no longer define a ground-truth using a complete eigen-decomposition. Alternative ``ground-truths'' supplied in \citep{mislove-2007-socialnetworks}, regarding each set of connected components with more than 3 nodes as a community, are not universally accepted. This definition, along with that of self-declared group membership \citep{yang2015defining}, often leads to contradictions with our definition of a community. A notable example is the Orkut dataset, where the number of stated communities is greater than the number of nodes \citep{snapnets}. Beyond being impossible to learn such a value from the eigenspectra, if the main reason to learn about clusters is to partition groups and to summarise networks into smaller substructures, such a definition is undesirable.

We present our findings for the number of clusters in the DBLP ($n=317,080$), Amazon ($n=334,863$) and YouTube ($n=1,134,890$) networks \citep{snapnets} in Table \ref{table:largedata}, where we use a varying number of moments. We see that for both the DBLP and Amazon networks, the number of clusters $N_{c}$ seems to converge with increasing moments number $m$, whereas for YouTube such a trend is not visible. This can be explained by looking at the approximate spectral density of the networks implied by maximum entropy in FIG \ref{fig:bigdata}. For both DBLP and Amazon (FIG \ref{fig:DBLP100moments} and \ref{fig:amazon100moments} respectively), we see that our method implies a clear spectral gap near the origin, indicating the presence of clusters. Whereas for the YouTube dataset, shown in FIG \ref{fig:youtube100moments}, no such clear spectral gap is visible and hence the number of clusters cannot be estimated accurately. 

\begin{figure}[H]
	\centering
	\vspace{-0.5cm}
	\begin{subfigure}{0.33\linewidth}
		\centering
	    \includegraphics[trim=0cm 0.cm 0.cm 0.cm, clip, width=1.0\linewidth]{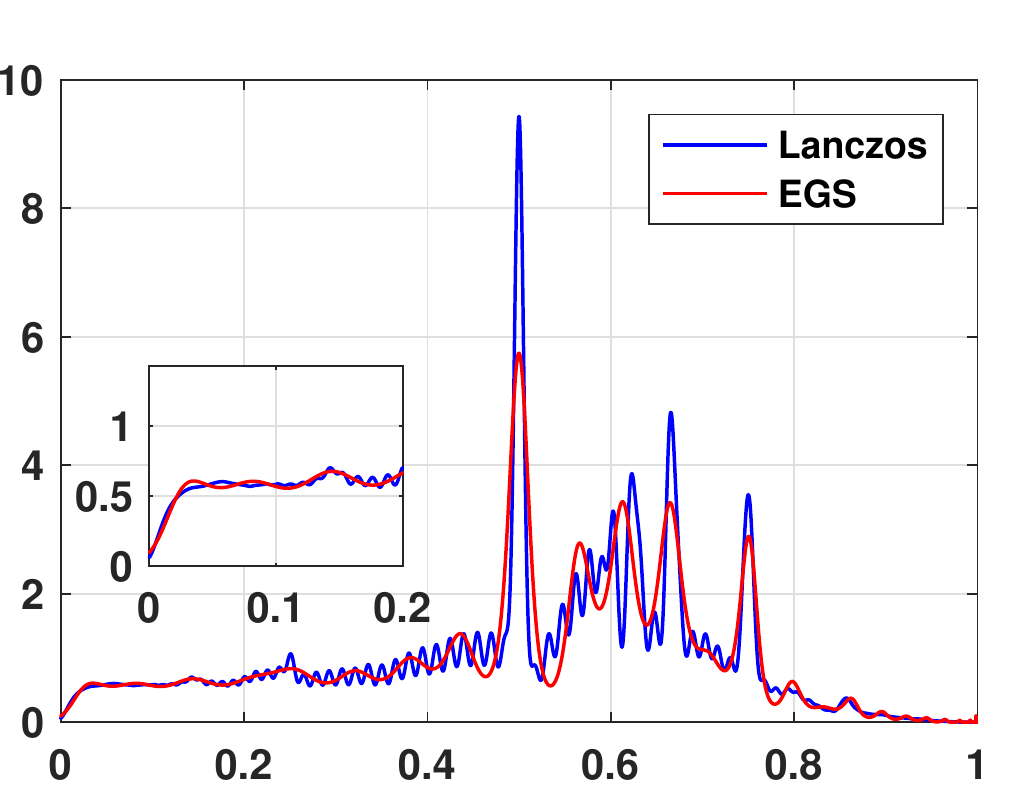}
	    \caption{DBLP}
	    \label{fig:DBLP100moments}	
	\end{subfigure}%
	\begin{subfigure}{0.33\linewidth}
		\centering
	    \includegraphics[trim=0cm 0.cm 0.cm 0.cm, clip, width=1.0\linewidth]{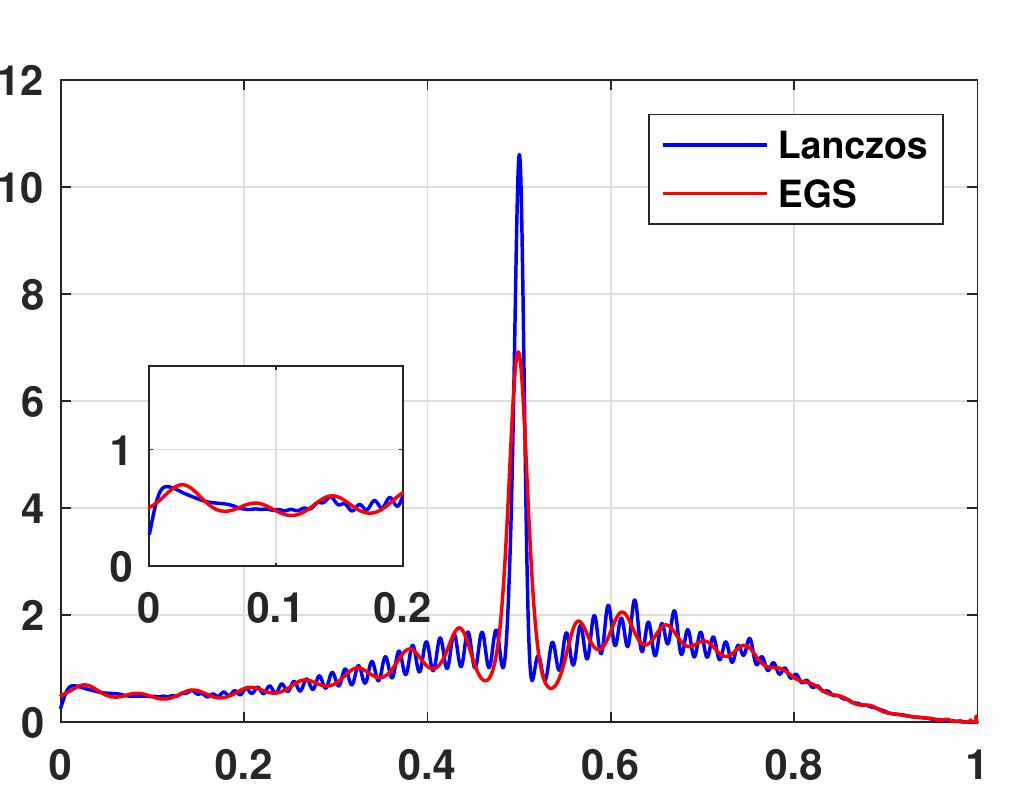}
	    \caption{Amazon}
    	\label{fig:amazon100moments}
    \end{subfigure} 
	\begin{subfigure}{0.33\linewidth}
    	\includegraphics[trim=0cm 0.cm 0.cm 0.cm, clip, width=1.0\linewidth]{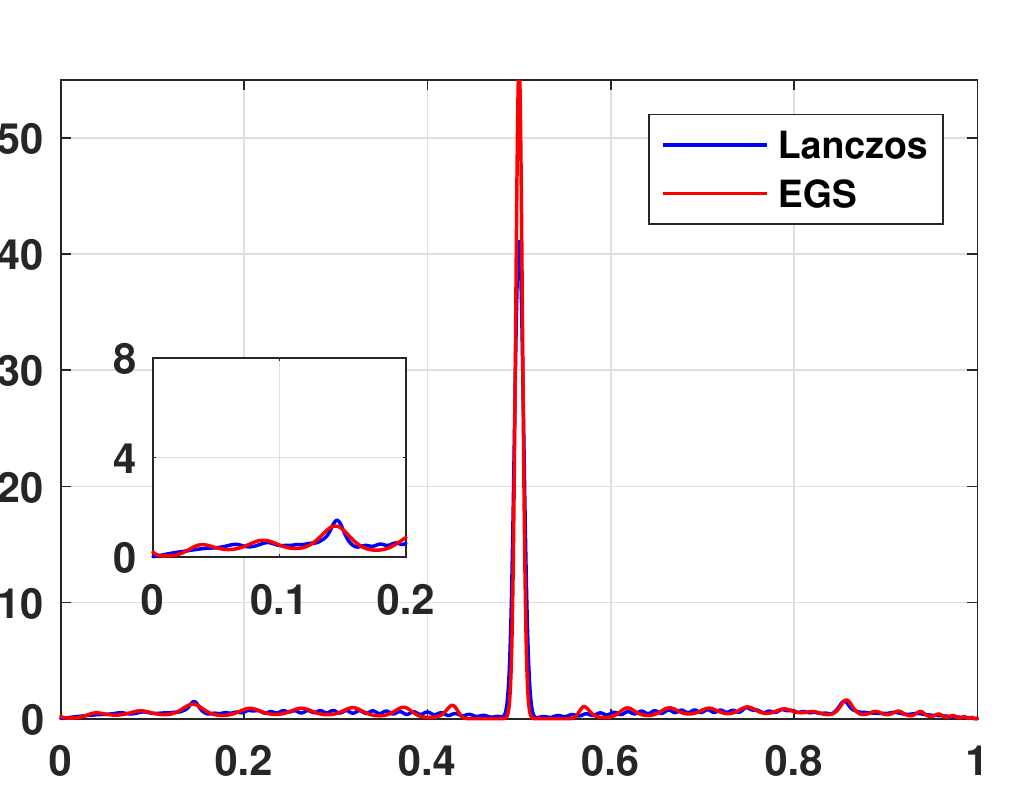}
    	\caption{YouTube}
    	\label{fig:youtube100moments}	
	\end{subfigure}%
	\caption{Spectral densities for DBLP, Amazon and YouTube datasets.}\label{fig:bigdata}
	\vspace{-0.5cm}
\end{figure} 

% implied by maximum entropy in Figure \ref{fig:bigdata}. For both DBLP and Amazon (Figure \ref{fig:DBLP100moments} and \ref{fig:amazon100moments} respectively), we see that our method implies a clear spectral gap near the origin, indicating the presence of clusters. Whereas for the YouTube dataset, shown in Figure \ref{fig:youtube100moments}, no such clear spectral gap is visible and hence the number of clusters cannot be estimated accurately. 

\begin{table}[t]
	\caption{Cluster number detection by EGS for DBLP, Amazon and YouTube . }	\label{table:largedata}
	\centering
	\begin{tabular}{lcccr}
	\hline
		\toprule
		Moments  & 40  & 70  & 100 \\
		\midrule
		DBLP ($n=317,080$)    & $ 2.215\times 10^{4}$   & $8.468 \times 10^{3}$  & $8.313\times 10^{3}$   \\
		Amazon ($n=334,863$) & $2.351\times 10^{4}$   & $1.146\times 10^{4}$   & $1.201\times 10^{4}$   \\
		Youtube ($n=1,134,890$)& $4.023\times 10^{3}$   & $1.306\times 10^{4}$   & $1.900\times 10^{4}$   \\
		\bottomrule
	\end{tabular}
\end{table}

% \begin{table}[h]
% 	\caption{Cluster number detection by MaxEnt for DBLP ($n=317,080$), Amazon ($n=334,863$) and YouTube ($n=1,134,890$). }	\label{table:largedata}
% 	\begin{center}
% 		\begin{small}
% 			\begin{sc}
% 				\begin{tabular}{|l|c|c|c|r}
% 				\hline
% 					\toprule
% 					Moments  & 40  & 70  & 100 \\
% 					\hline
% 					\midrule
% 					DBLP    & $ 2.215\times 10^{4}$   & $8.468 \times 10^{3}$  & $8.313\times 10^{3}$   \\
% 					Amazon & $2.351\times 10^{4}$   & $1.146\times 10^{4}$   & $1.201\times 10^{4}$   \\
% 					Youtube & $4.023\times 10^{3}$   & $1.306\times 10^{4}$   & $1.900\times 10^{4}$   \\
% 					\bottomrule
% 					\hline
% 				\end{tabular}
% 			\end{sc}
% 		\end{small}
% 	\end{center}
% 	\vskip -0.1in
% \end{table}

\section{Conclusion}
In this paper, we propose a novel, efficient framework for learning a continuous approximation to the spectrum of large scale graphs, which overcomes the limitations introduced by kernel smoothing. We motivate the informativeness of spectral moments using the link between random graph models and random matrix theory. We show that our algorithm is able to learn the limiting spectral densities of random graph models for which analytical solutions are known. 
We showcase the strength of this framework in two real world applications, namely, computing the similarity between different graphs and detecting the number of clusters in the graph. Interestingly, we are able to classify different real world networks with respect to their similarity to classical random graph models. The EGS may be of further use to researchers studying network properties and similarities. 

\newpage
\appendix

\section{Stochastic Trace Estimation} \label{subset: stoctrace}

The intuition behind stochastic trace estimation is that we can accurately approximate the moments of $\lambda$ with respect to the spectral density $p(\lambda)$ by using computationally cheap matrix-vector multiplications. The moments of $\lambda$ can be estimated using a Monte-Carlo average,
\begin{equation}
    n\,\mathbb{E}_{p}(\lambda^{m}) =\mathbb{E}_{\mvec{v}}(\mvec{v}^{T} \mvec{X}^{m}\mvec{v}) \approx \frac{1}{d}\sum_{j=1}^{d}\mvec{v}_{j}^{T}\mvec{X}^{m}\mvec{v}_{j},
\end{equation}
where $\mvec{v}_j$ is any random vector with zero mean and unit covariance and $\mvec{X}$ is a $n \times n$ matrix whose eigenvalues are $\{ \lambda_i \}_{i=1}^n$. This enables us to efficiently estimate the moments in $\mathcal{O}(d\times m\times n_\mathrm{nz})$ for sparse matrices, where $d\times m \ll n$. We use these as moment constraints in our entropic graph spectrum (EGS) formalism to derive the functional form of the spectral density. Examples of this in the literature include \citep{ubaru2017fast,ete}.

\begin{algorithm}[H]
	\caption{Learning the Graph Laplacian Moments.}
	\label{alg:preprocessing}
	\begin{algorithmic}[1]
		\STATE {\bfseries Input:} Normalized Laplacian $\mvec{L}_{\mathrm{norm}}$, Number of Probe Vectors $d$, Number of moments required $m$
		\STATE {\bfseries Output:} Moments of Normalised Laplacian $\{\mu_{i}\}$
		\FOR {$i$ in $1,..,d$}
			\STATE  Initialise random vector $\mvec{z}_{i}\in R^{1\times n}$
			\FOR {$j$ in $1,..,m$}
				\STATE $\mvec{z}_{i}' = \mvec{L}_{\mathrm{norm}}\mvec{z}_{i}$
				\STATE $\rho_{ij} =  \mvec{z}_{i}\Trans \mvec{z}_{j}'$
				%\STATE $j = j+1$
			\ENDFOR
			%\STATE $$
		\ENDFOR
		\STATE $\mu_{i} = 1/d \times \sum_{j=1}^{d}\rho_{ij}$ 
		%\UNTIL{$noChange$ is $true$}
	\end{algorithmic}
\end{algorithm}

\section{Comment on the Lanczoz Algorithm}\label{subsec: lanczos}
In the state-of-the-art iterative algorithm Lanczos \citep{ubaru2017fast}, the tri-diagonal matrix  $\mvec{T}^{m\times m}$ can be derived from the moment matrix $\mvec{M}^{m\times m}$, corresponding to the discrete measure $d\alpha(\lambda)$ satisfying the moments $\mu_{i} = v^{T}X^{i}v = \int \lambda^{i}d\alpha(\lambda)$ for all $i \leq m$ \citep{golub1994matrices} and hence it can be seen as a weighted Dirac approximation to the spectral density matching the first $m$ moments.
The weight given on every Ritz eigenvalue $\lambda''_{i}$ (the eigenvalues of the matrix $\mvec{T}^{m\times m}$) is the square of the first component of the corresponding eigenvector, i.e., ${[\phi_{i}]_1}^{2}$, hence the approximated spectral density can be written as,
\begin{equation}
    \frac{1}{n}\sum_{i}^{n}\delta(\lambda-\lambda_{i}) \approx \sum_{i}^{m}w_{i}\delta(\lambda-\lambda''_{i}) =  \sum_{i}^{m}\phi_{i}[1]^{2}\delta(\lambda-\lambda''_{i}).
\end{equation}

\section{Experimental Details} \label{subsec:implementation_details}
% \label{experiments}
We use $d=100$ Gaussian random vectors for our stochastic trace estimation, for both EGS and Lanczos \citep{ubaru2017fast}. We explain the procedure of going from adjacency matrix to Laplacian moments in Algorithm \ref{alg:preprocessing}. When comparing EGS with Lanczos, we set the number of moments $m$ equal to the number of Lanczos steps, as they are both matrix vector multiplications in the Krylov subspace. We further use Chebyshev polynomial input instead of power moments for improved performance and conditioning. In order to normalise the moment input we use the normalised Laplacian with eigenvalues bounded by $[0,2]$ and divide by $2$. %We then apply our cluster number estimator, Algorithm \ref{alg:clusteralg} to both the spectral density derived from our MaxEnt implementation and to that implied by the Lanczos algorithm. 
To make a fair comparison we take the output from Lanczos \citep{ubaru2017fast} and apply kernel smoothing \citep{lin2016approximating} before applying our cluster number estimator.

\begin{figure}[t]
    \centering
    \includegraphics[width = 0.8\linewidth]{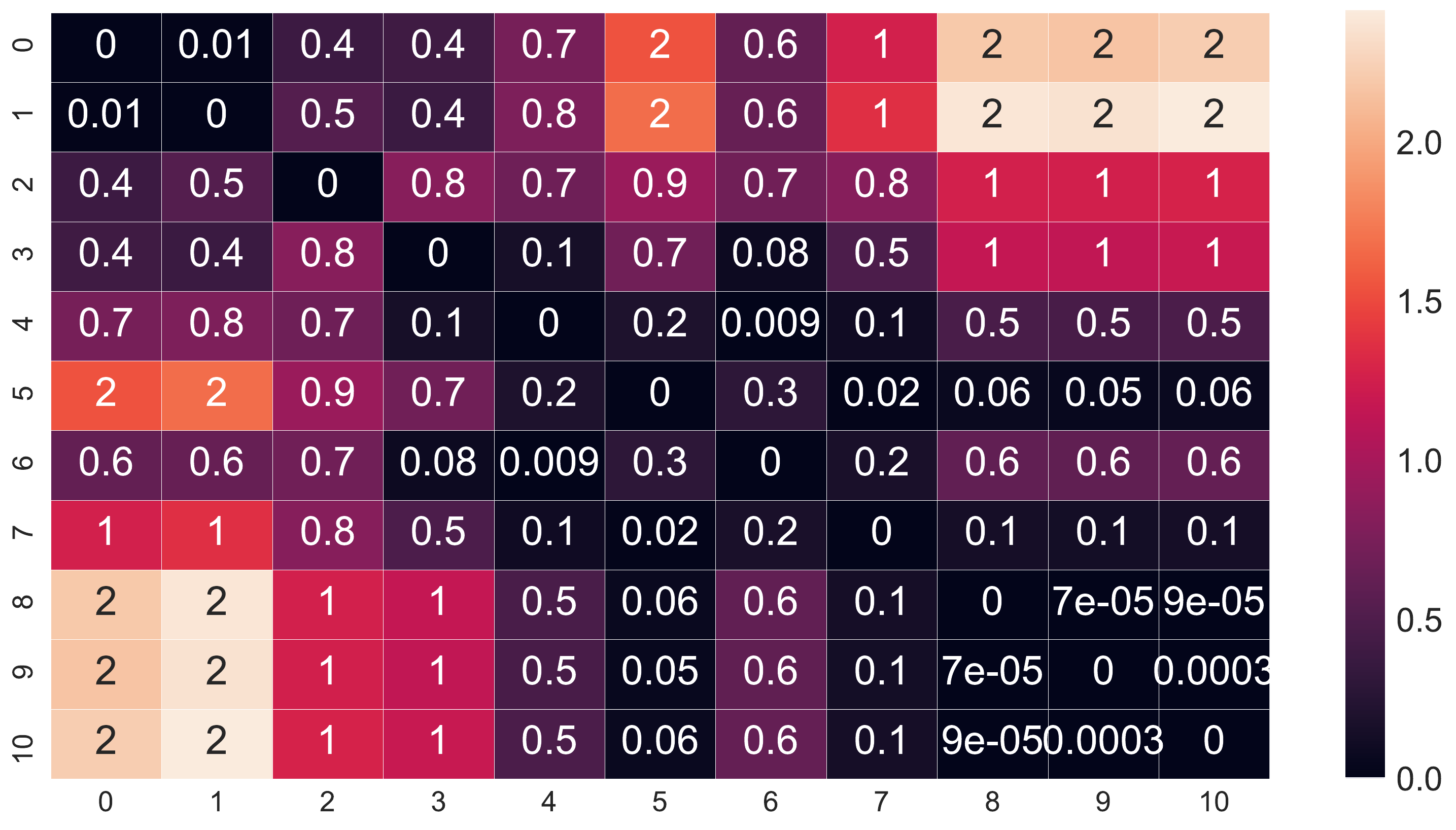}
    \caption{Symmetric KL heatmap, obtained using only $3$ moments, i.e., Gaussian approximation, between 9 graphs from the SNAP dataset: (0) bio-human-gene1,
(1) bio-human-gene2,
(2) bio-mouse-gene,
(3) ca-AstroPh,
(4) ca-CondMat,
(5) ca-GrQc,
(6) ca-HepPh,
(7) ca-HepTh,
(8) roadNet-CA,
(9) roadNet-PA,
(10) roadNet-TX.}
    \label{fig:graphsinthewild3moments}
\end{figure}

\begin{figure}[t]
    \centering
    \includegraphics[width = 0.8\linewidth]{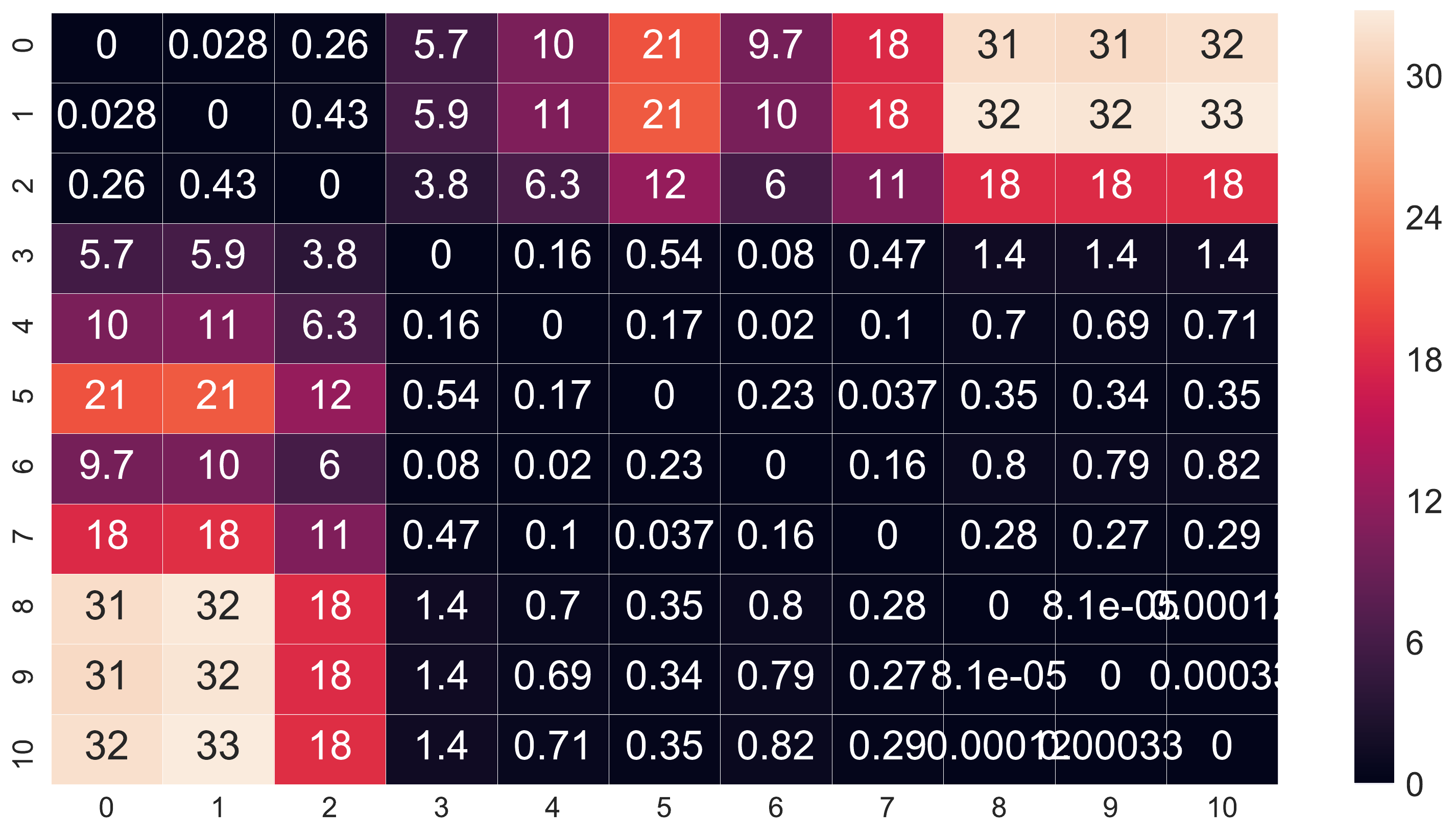}
    \caption{Symmetric KL heatmap, obtained using only $8$ moments, between 9 graphs from the SNAP dataset: (0) bio-human-gene1,
(1) bio-human-gene2,
(2) bio-mouse-gene,
(3) ca-AstroPh,
(4) ca-CondMat,
(5) ca-GrQc,
(6) ca-HepPh,
(7) ca-HepTh,
(8) roadNet-CA,
(9) roadNet-PA,
(10) roadNet-TX.}
    \label{fig:graphsinthewild8moments}
\end{figure}

\section{EGSs of Real World Networks with Varying Number of Moments}
In order to more clearly showcase the practical value of having a EGS based on a large number of moments, we show the symmetric KL divergence between real world networks using a $3$ moment Gaussian approximation. The Gaussian is fully defined by its normalization constant, mean and variance and so can be specified with $m=3$ Lagrange multipliers. The results for the same analysis as in Figure \ref{fig:graphsinthewild}, but now obtained using a $3$ moment Gaussian approximation, are shown in Figure \ref{fig:graphsinthewild3moments}. The networks are still somewhat distinguished; however, one can see for example that citation networks and road networks are less clearly distinguished to the point that inter-class distance is lessened compared to intra-class distance, which for the purpose of network classification is not a particularly helpful property. 
% further as shown in the main text in Figure \ref{fig:graphsinthewild} this is not the case for more moments. 
The problem still persists for more moments; for example, when we choose $m=8$, which is what has been reported stable for other off-the-shelf maximum entropy algorithms, 
% but similar results apply for $m=10$ or $15$ moments, which is shown in Figure \ref{fig:graphsinthewild8moments}.
similar results are observed in Figure \ref{fig:graphsinthewild8moments}.
In comparison, this is not the case for more moments in Figure \ref{fig:graphsinthewild} in the main text.

\section{On the Importance of Moments} 
\label{sec:momentsmatter}
Given that all iterative methods essentially generate a $m$ moment empirical spectral density (ESD) approximation, it is instructive to ask what information is contained within the first $m$ spectral moments. %In the ideal case we would like to be able to garner information about the underlying stochastic process. Put simply, if we have two processes of friendship generation $A$ and $B$, creating graphs $G_{A}$ and $G_{B}$, can the moments of graphs $G_{A}$ and $G_{B}$

To answer this question concretely, we consider the spectra of random graphs. By investigating the finite size corrections and convergence of individual moments of the empirical spectral density (ESD) compared to those of the limiting spectral density (LSD), we see that the observed spectra are faithful to those of the underlying stochastic process. Put simply, 
% if we have a friendship model, which satisfies the conditions of a random graph and we compare the moments of the spectral density of a single instance of that graph to the average over many instances. I.e The moments we observe are informative about the underlying stochastic process.
given a random graph model, if we compare the moments of the spectral density observed from a single instance of the model to that averaged over many instances, we see that the moments we observe are informative about the underlying stochastic process.

\subsubsection{ESD moments converge to those of the LSD}
For random graphs, with independent edge creation probabilities, their spectra can be studied through the machinery of random matrix theory \citep{akemann2011oxford}.
%Consider the spectral moments of a given random matrix $\langle \mu_{n},x^{m}\rangle$ and their expectation under the empirical density, $\langle \hat{\mu}_{n},x^{m}\rangle$. Here $\mu_{n}$ is the true density, $\hat{\mu}_{n}$ the empirical density and $\langle \mu_{n},x^{m}\rangle = \int_{\mathcal{R}}x^{m}d\mu(x)$ the inner product. 
%Using Chebyshev's inequality

%\begin{eqnarray}
    %\begin{aligned}
%    &&P\big(|\langle \mu_{n},x^{m}\rangle - \langle \hat{\mu}_{n},x^{m}\rangle |>\epsilon\big)\nonumber\\ 
%    \quad \leq &&\frac{\big|\mathbb{E}\big(\langle \mu_{n},x^{m}\rangle \big)^{2}-\big(\mathbb{E}\langle \mu_{n},x^{m}\rangle\big)^{2}\big|}{\epsilon^{2}} 
    %\end{aligned}
%\end{eqnarray}
%\xd{notations in (6) need to be defined} Writing this in terms of matrix traces, we have

%\begin{equation} \label{eq:rmtgraph}
%    \frac{1}{n^{2}}\big[\mathbb{E}\big((\text{Tr}X_{n}^{m})^{2}-(\mathbb{E}\text{Tr}X^{m}_{n}\big)^{2}\big] = \frac{1}{n^{2}}\sum_{i,i'}\big[\mathbb{E}\zeta_{i}\zeta_{i'}-\mathbb{E}\zeta_{i}\mathbb{E}\zeta_{i'}\big]
%\end{equation}
%where $\zeta_{i}$ is shorthand for the product $\zeta_{i_{1}i_{2}} \times ... \times \zeta_{i_{m}i_{1}}$ \xd{$\zeta_{i_{1}i_{2}}$ etc. needs to be defined} with $i_{1},...,i_{m} \in \{1,...,n\}$, where each pair $i,i'$ generates a graph with vertices $V_{i,i'} = \{i_{1},..i_{m}\} \cup \{i'_{1},...,i'_{m}\}$ and edges $E_{i,i'} = \{i_{1}i_{2},..,i_{m}i_{1}\} \cup \{i'_{1}i'_{2},....,i'_{m}i'_{1}$\}.  

We consider the entries of an $n\times n$ matrix $\mathbf{X}_n$ to be zero mean and independent, with bounded moments. For such a matrix, a natural scaling which ensures we have bounded norm as $n \rightarrow \infty$ is $\mathbf{X}_{n} = \mathbf{M}_{n}/\sqrt{n}$.  
%For the term in equation \eqref{eq:rmtgraph} corresponding to ${i,i'}$ to be non-zero, each edge must appear at least twice, since the entries of $X_{n}$ have zero mean and the graphs must have at least one edge in common otherwise equation \eqref{eq:rmtgraph} is $0$ by independence. The {\it{weight}} of $(i,i')$ is defined as the cardinality of $V_{i,i'}$ with pairs $(i,i')$ and $(j,j')$ said to be equivalent if there is a bijection mapping coresponding indices to each other. It can be easily shown that there are no non-zero pairs of weights $t \geq m +1$. 
%It can be shown that the maximum weight within the sum on the right hand side of (\ref{eq:rmtgraph}) is $\mathcal{O}(n^{m})$ and, by the boundedness of the moments of the entries of $M_{n}=\sqrt{n} X_{n}$, the contribution of each term is $\mathcal{O}(1/n^{m+2})$  \citep{feier2012methods}. So 
It can be shown (see for instance \citep{feier2012methods}) that the moments of a particular instance of a random graph and the related random matrix $\mathbf{X}_n$ converge to those of the limiting counterpart in probability with a correction of $\mathcal{O}(n^{-2})$.

\subsubsection{Finite size corrections to moments get worse with larger moments}
 A key result, akin to the normal distribution for classical densities, is the semi-circle law for random matrix spectra \citep{feier2012methods}. For matrices with independent entries $a_{ij}$, $\forall i>j$, with common element-wise bound $K$, common expectation $\mu$ and variance $\sigma^{2}$, and diagonal expectation $\mathbb{E}a_{ii}=\nu$, it can be shown that the corrections to the semi-circle law for the moments of the eigenvalue distribution,
\begin{equation}
 %\langle \mu_{n},\lambda^{m}\rangle =
 \int x^{m}d\mu(x) = \frac{1}{n}\text{Tr}\mathbf{X}^{m}_{n},
\end{equation}
have a corrective factor bounded by \citep{furedi1981eigenvalues}
\begin{equation}
    \frac{K^{2}m^{6}}{2\sigma^{2}n^{2}}.
\end{equation}
Hence, the finite size effects are larger for higher moments than that for the lower counterparts. This is an interesting result, as it means that for large graphs with $n \rightarrow \infty$, the lowest order moments, which are those learned by any iterative process, best approximate those of the underlying stochastic process.

\section{Analytic Forms for the Differential Entropy and divergence from EGS}\label{sec: analytic_diff_ent_and_divergence}

To calculate the differential entropy we simply note that
\begin{equation}
    \mathcal{S}(p) = \int p(\lambda) (1+\sum_{i}^{m}\alpha_{i}\lambda^{i})d\lambda = 1+\sum_{i}^{m}\alpha_{i}\mu_{i}.
\end{equation}
\noindent The KL divergence between two EGSs, $p(\lambda) = \exp[-(1+\sum_{i}\alpha_{i}\lambda^{i})]$ and $q(\lambda) = \exp[-(1+\sum_{i}\beta_{i}\lambda^{i})]$, can be written as,
\begin{equation}
    \mathcal{D}_{\mathrm{KL}}(p||q) = \int p(\lambda)\log \frac{p(\lambda)}{q(\lambda)}d\lambda = -\sum_{i}(\alpha_{i}-\beta_{i})\mu_{i}^{p},
\end{equation}
where $\mu_{i}^{p}$ refers to the $i$-th moment constraint of the density $p(\lambda)$. Similarly, the symmetric-KL divergence can be written as,
\begin{equation}
    \frac{\mathcal{D}_{\mathrm{KL}}(p||q)+\mathcal{D}_{\mathrm{KL}}(q||p)}{2} = \frac{\sum_{i}(\alpha_{i}-\beta_{i})(\mu_{i}^{q}-\mu_{i}^{p})}{2},
\end{equation}
where all the $\alpha$ and $\beta$ are derived from the optimisation and all the $\mu$ are given from the stochastic trace estimation.

%MaxEnt is not only theoretically grounded, but will yield a computationally cheaper method for density estimation, as established in Section \ref{subset: stoctrace}.  
%We also find that the distribution implied by the method of maximum entropy faithfully represents the true density and the well-established Lanczos approximate density, as shown in Figure \ref{fig:emaildensity}. We note that the bulk of the distribution is better approximated by the MaxEnt method.

\section{Spectral Density with More Moments}

We display the process of spectral learning for both EGS and Lanczos, by plotting the spectral density of both methods against the ground-truth in FIG \ref{fig:emaildensity}. In order to make a valid comparison, we smooth the implied density using a Gaussian kernel with $\sigma = 10^{-3}$. Whilst this number could in theory be optimised over, we considered a range of values and took the smallest for which the density was sufficiently smooth, i.e., everywhere positive on the bounded domain $[0,1]$. We note that both EGS and Lanczos approximate the ground-truth better with a greater number of moments $m$ and that Lanczos learns the extrema of the spectrum before the bulk of the distribution while EGS captures the bulk right from the start.

\begin{figure}[t]
	\centering
	\includegraphics[trim=0cm 0cm 0.0cm 0.0cm, clip, width=1.0\linewidth]{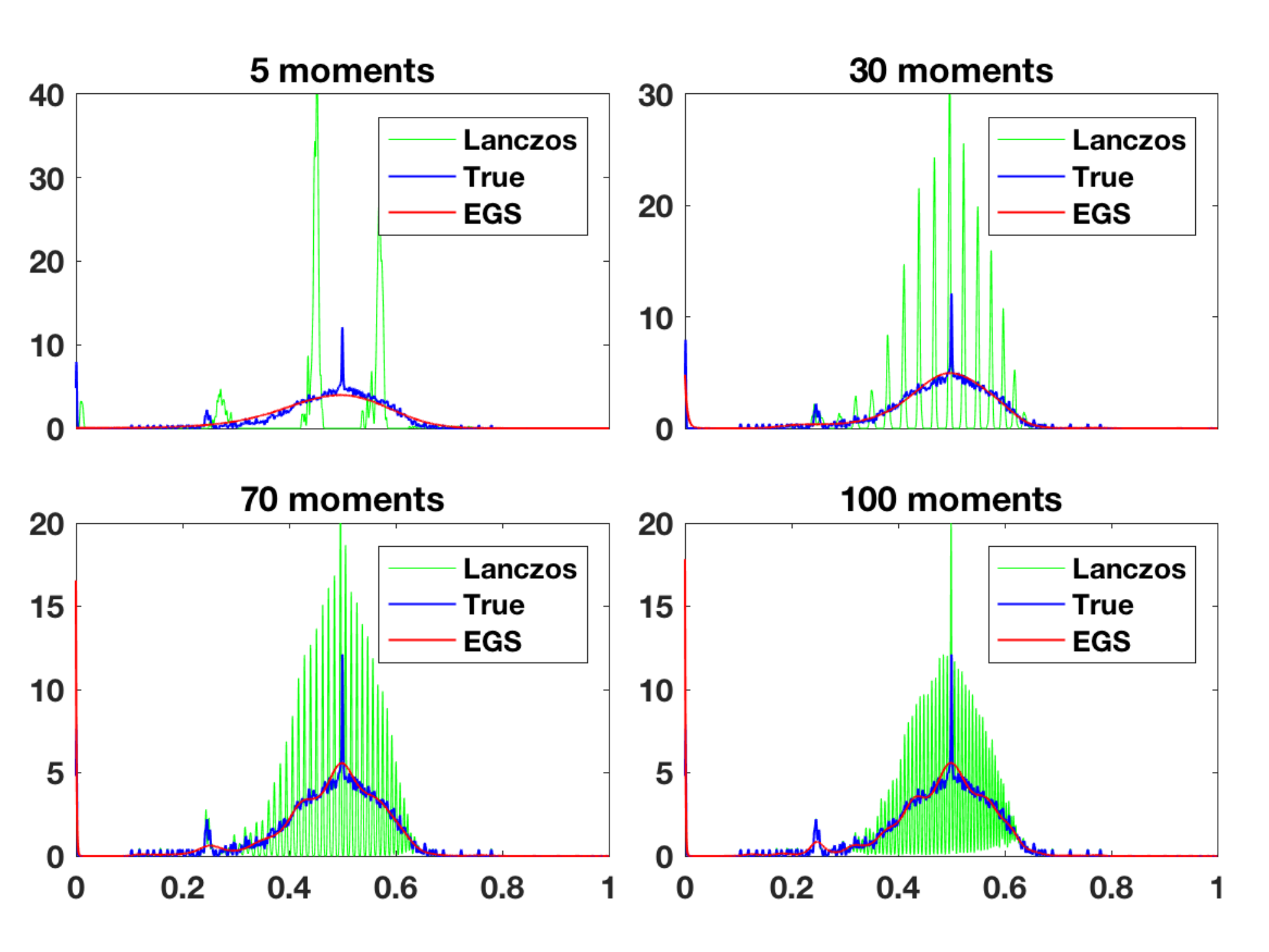}
	\caption{Spectral density for varying number of moments $m$, for both EGS and Lanczos algorithms as well as the ground-truth.}
	\label{fig:emaildensity}	
\end{figure} 

\bibliographystyle{apsrev4-1}
\bibliography{bibi}

\end{document}